\documentclass[dvipsnames,format=sigconf,anonymous=false,review=false]{acmart}
\usepackage{times} 
\usepackage{helvet} 
\usepackage{courier} 
\usepackage{graphicx}
\usepackage{natbib}
\usepackage{caption}
\usepackage[linesnumbered,titlenumbered,ruled,vlined,resetcount]{algorithm2e}
\SetKwComment{Comment}{$\triangleright$\ }{}
\usepackage{newfloat}
\usepackage{listings}

\DeclareCaptionStyle{ruled}{labelfont=normalfont,labelsep=colon,strut=off}
\floatstyle{ruled}
\newfloat{listing}{tb}{lst}{}
\floatname{listing}{Listing}

\usepackage{booktabs}       
\usepackage{amsfonts}       
\usepackage{nicefrac}       
\usepackage{microtype}      
\usepackage{xcolor}         
\usepackage{amsmath}
\usepackage{amsthm}
\usepackage{xspace}
\usepackage{caption}
\usepackage{svg}
\usepackage{float}
\usepackage{dsfont}
\usepackage{placeins}
\usepackage[capitalize]{cleveref}

\allowdisplaybreaks
\newcounter{relctr} 
\everydisplay\expandafter{\the\everydisplay\setcounter{relctr}{0}} 

\newcommand\labelrel[2]{%
  \begingroup
    \refstepcounter{relctr}%
    \stackrel{\textnormal{(\alph{relctr})}}{\mathstrut{#1}}%
    \originallabel{#2}%
  \endgroup
}
\AtBeginDocument{\let\originallabel\label}

\definecolor{green}{HTML}{FF7F00}
\newtheorem{theorem}{Theorem}
\newtheorem*{remark}{Remark}

\newcommand{\E}{\ensuremath{\mathbb{E}}\xspace}
\newcommand{\R}{\ensuremath{\mathbb{R}}\xspace}
\newcommand{\N}{\ensuremath{\mathcal{N}}\xspace}

\newcommand{\vtheta}{\ensuremath{\vec{\theta}}\xspace}

\newcommand{\loss}{\ensuremath{\mathcal{L}}\xspace}
\newcommand{\nloss}{\ensuremath{\tilde{\loss}}\xspace}
\newcommand{\mloss}{\ensuremath{\bar{\loss}}\xspace}

\newcommand{\vx}{\ensuremath{\vec{x}}\xspace}

\newcommand{\vz}{\ensuremath{\vec{z}}\xspace}
\newcommand{\vepsilon}{\ensuremath{\vec{\varepsilon}}\xspace}

\newcommand{\ud}{\operatorname{\mathrm{d}}\xspace}

\newcommand{\norm}[1]{\ensuremath{\left\|#1\right\|}\xspace}
\newcommand{\convd}{\ensuremath{\xrightarrow{d}}\xspace}
\AtBeginDocument{%
  }

\setcopyright{rightsretained}
\acmDOI{10.1145/3712256.3726352}
\acmISBN{979-8-4007-1465-8/2025/07}
\acmConference[GECCO '25]{Genetic and Evolutionary Computation Conference}{July 14--18, 2025}{Malaga, Spain}
\acmYear{2025}
\copyrightyear{2025}

\begin{document}

\title{An Adaptive Re-evaluation Method for Evolution Strategy under Additive Noise}

\author{Catalin-Viorel Dinu}
\orcid{0009-0008-8490-9946}
\affiliation{%
  \institution{LIACS, Leiden University}
  \city{Leiden}
  \country{The Netherlands}}
\email{viorel.dinu00@gmail.com}

\author{Yash J. Patel}
\orcid{0009-0007-3060-6950}
\affiliation{%
  \institution{applied Quantum algorithms, Leiden University}
  \institution{LIACS, Leiden University}
  \city{Leiden}
  \country{The Netherlands}
}
\email{y.j.patel@liacs.leidenuniv.nl}

\author{Xavier Bonet-Monroig}
\orcid{0000-0003-4895-4180}
\affiliation{%
\institution{Honda Research Institute Europe GmbH \\ Frankfurt, Germany}
\institution{Instituut-Lorentz, Leiden University}
  \city{Leiden}
  \country{The Netherlands}}
\email{xavier.bonet@honda-ri.de}

\author{Hao Wang}
\orcid{0000-0002-4933-5181}
\affiliation{%
  \institution{applied Quantum algorithms, Leiden University}
  \institution{LIACS, Leiden University}
  \city{Leiden}
  \country{The Netherlands}}
\email{h.wang@liacs.leidenuniv.nl}

\renewcommand{\shortauthors}{Dinu et al.}



\begin{abstract}
The Covariance Matrix Adaptation Evolutionary Strategy (CMA-ES) is one of the most advanced algorithms in numerical black-box optimization. For noisy objective functions, several approaches were proposed to mitigate the noise, e.g., re-evaluations of the same solution or adapting the population size.
In this paper, we devise a novel method to adaptively choose the optimal re-evaluation number for function values corrupted by additive Gaussian white noise. We derive a theoretical lower bound of the expected improvement achieved in one iteration of CMA-ES, given an estimation of the noise level and the Lipschitz constant of the function's gradient.
Solving for the maximum of the lower bound, we obtain a simple expression of the optimal re-evaluation number.
We experimentally compare our method to the state-of-the-art noise-handling methods for CMA-ES on a set of artificial test functions across various noise levels, optimization budgets, and dimensionality. Our method demonstrates significant advantages in terms of the probability of hitting near-optimal function values.
\end{abstract}




\maketitle

\section{Introduction}
Optimization problems are central to various scientific and engineering fields~\cite{kochenderfer2019algorithms,martins2021engineering}.
Typically, these problems are analyzed under ideal conditions, assuming a noiseless environment. However, many real-world optimization problems involve noise, which can distort the true objective function, making the optimization landscape less reliable.
Noise in the objective function can arise from various sources, such as measurement errors, environmental variability, or the inherent randomness of the system \cite{rakshit2017noisy}.
In response, several noise models have been proposed in the literature, which can broadly be categorized into two types:
\begin{itemize}
    \item \textbf{Additive Noise} \cite{dang2015efficient, rowe2021evolutionary}: This model assumes that the noise added to the objective function value is independent of the function value itself. It is expressed as $\nloss(\vx) = \loss(\vx) + \tau\mathcal{N}(0,1)$, where $\tau$ represents the standard deviation of the noise.
    \item \textbf{Multiplicative Noise} \cite{uchida2024cma}: In this model, the noise scales with the objective function value.
It is expressed as $\nloss(\vx) = (1+\tau z)\loss(\vx)$, where $z$ can be a Gaussian or uniform random variable.
\end{itemize}

In noisy black-box optimization, evolutionary algorithms (EAs) have shown promising performance due to the intrinsic population dynamics, which improves the robustness against noise~\cite{arnold2002noisy, BOS04, hansen2008method, rakshit2017noisy}.
The Covariance Matrix Adaptation Evolution Strategy (CMA-ES)~\cite{Hansen16a} is the state-of-the-art algorithm among EAs~\cite{varelas2018comparative}.
To further improve the capabilities of CMA-ES over noisy evaluations, several noise-handling methods have been proposed, of which the most popular are:
\begin{itemize}
    \item \textbf{Population size adaptation}~\cite{735433, harik1999gambler, 9647022} involves dynamically modifying the population size to mitigate the noise in function values.
    Using a larger population increases the probability of selecting candidate points whose fitness values are closer to the noiseless value.
    \item \textbf{Learning rate adaptation}~\cite{nomura2023cma} adjusts the learning rate according to the noise level. Smaller steps can reduce sensitivity to noise, leading to steady and progressive gains toward the optimum.
    \item \textbf{Re-evaluating of the objective function} is the most common noise-handling method.
    for each candidate multiple times and then take the average to reduce the noise effect~\cite{aizawa1993dynamic, 6791253, hansen2008method, bonet2023performance, SRLR12}.
    This approach helps smooth out artificial fluctuations in the function landscape induced by the noise. 
\end{itemize}

In this work, we devise a novel method to determine the optimal re-evaluation number for each candidate point. For this, we consider objective functions with Lipschitz continuous gradient and an additive Gaussian noise model and derive a lower bound on the expected improvement of the function values at each iteration of CMA-ES.
In turn, the analytical bound gives us a simple expression of the optimal number of re-evaluations for each candidate of the CMA-ES iteration.
We implement this method into CMA-ES, an extension that we call \textbf{Adaptive Re-evaluation} method (AR-CMA-ES).
To benchmark our method, we use a wide range of test functions with different levels of noise of the objectives.
Our experimental results show that AR-CMA-ES outperforms existing noise-handling methods at all noise levels, achieving a much better accuracy-to-target across the test benchmarks.
To summarize our contributions,
\begin{itemize}
    \item we have derived a theoretical lower bound of the expected improvement of noiseless function values in one iteration of CMA-ES regardless of the objective function;
    \item we have chosen the optimal re-evaluation number by maximizing the efficiency metric, which is the expected improvement normalized by the re-evaluations; 
    \item we have obtained a simple analytical expression for the optimal re-evaluation number and provide estimation procedures for the parameters required by the expression. 
\end{itemize}

\section{Background}\label{sec:preliminary}
\paragraph{Problem formulation:}
We aim to minimize a single-objective, black-box, differentiable function $\loss: \R^d \rightarrow \R$.
In this study, we specifically address the scenario involving additive Gaussian noise, noting that similar analytical approaches can be applied to other types of noise.
The noisy function value is represented as:
\begin{equation*}
    \nloss(\vx)= \loss(\vx) + \tau\mathcal{N}(0, 1).
\end{equation*}
We assume that the gradient of the function $\loss$ is Lipschitz continuous, meaning $\exists K<\infty$ such that $\norm{\nabla\loss(\vx) - \nabla\loss(\vx')}_2 \leq K \norm{\vx - \vx'}_2$ for all $\vx,\vx'\in\R^d$.

The re-evaluation method estimated $\loss(\vx)$ through the sample mean, is commonly used as a noise-mitigation method.
According to the Central Limit Theorem (CLT), we have $\sqrt{M}(\loss(\vx) - \mloss(\vx)) \convd \tau\mathcal{N}(0,1)
$, where $\mloss(\vx) = M^{-1}\sum_{i=1}^{M}y_i$ is computed by taking $M$ independent and identically distributed (i.i.d.) samples $y_i$ drawn from $\mloss(\vx)$.

Determining the appropriate value of $M$ is crucial.
Ideally, $M$ should be large enough to ensure that the re-evaluated estimates $\mloss(\vx^i)$ and $\mloss(\vx^j)$ can be distinguished with high probability, i.e., $\tau/\sqrt{M} \ll |\loss(\vx^i) - \loss(\vx^j)|.$
Considering that the set $\{\vx\,^i\}_i$ is contained within a compact subset of $\R^d$, we encounter the following two scenarios:

\begin{itemize}
    \item When $\loss$ exhibits a large local Lipschitz constant, a smaller $M$ is sufficient since $|\loss(\vx^i) - \loss(\vx^j)|$ is relatively large.
    \item When the Lipschitz constant is small, the difference $|\loss(\vx^i) - \loss(\vx^j)|$ is also small, requiring a much larger $M$ to ensure that the noise does not obscure these differences.
\end{itemize} 

\begin{algorithm}[t]
\caption{AR-CMA-ES. Our modifications to the standard CMA-ES are highlighted.}\label{alg:cma-es}
\textbf{Procedure:} AR-CMA-ES($\nloss$, $B$, $\lambda$, $\vec{x}_L, \vec{x}_U$)\;
\textbf{Input:} a noisy objective function $\nloss$, population size $\lambda$, evaluation budget $B$, $[\vec{x}_L, \vec{x}_U]\subseteq\R^d$\;
    $\sigma \leftarrow 0.1 \times\norm{\vec{x}_U -\vec{x}_L}_{\infty}$\Comment*[r]{\small step size}
    $\mathbf{C}\leftarrow \mathbf{I}$\Comment*[r]{\small covariance matrix}
    \textcolor{green}{$M \leftarrow 1, \vec{g} \leftarrow 0$}\;
    Sample $\vec{m}$ u.a.r.~in $[\vec{x}_L, \vec{x}_U]$\;
    \textcolor{green}{Estimate the noise level $\tau$}\;
    \Repeat{$B \leq 0$}{
        \For{$i\in[1..\lambda]$}{
            $\vx^{\,i} \leftarrow \vec{m} + \mathbf{C}^{1/2}\vec{\varepsilon}\,^i, \ \vec{\varepsilon}\,^i\sim \sigma\N(0, \mathbf{I})$\;
            $\mloss(\vx\,^i) \leftarrow \sum_{i=1}^M \nloss(\vx\,^i) / M$\;
            $\Delta\mloss(\vx^{\,^i}) \leftarrow \mloss(\vec{m}) - \mloss(\vx^{\,^i})$\;
        }
        \textcolor{green}{$A \leftarrow -\min \{\Delta\mloss(\vx\,^i)\}_i$}\;
        $B \leftarrow B - \lambda M$\;
        $\textcolor{green}{w_i \leftarrow \frac{\Delta\mloss^i + A}{\sum_{k=1}^\lambda\Delta\mloss^k + \lambda A}}$\Comment*[r]{Eq.~\eqref{eq:proportional-weight}}
        $\vec{m} \leftarrow \vec{m} + \sum_{i=1}^{\textcolor{green}{\lambda}} w_{i} \mathbf{C}^{1/2}\vec{\varepsilon}\,^i$\;
        $s_{\text{max}} \leftarrow$ the largest eigenvalue of $\mathbf{C}$\;
        \textcolor{green}{Estimate the Lipschitz constant $K$ of $\nabla \loss$}\;
        \textcolor{green}{$\vec{g} \leftarrow (1-\alpha)\vec{g} - \frac{\alpha}{\lambda\sigma^2}\sum_{i=1}^\lambda (\Delta \mloss^{i} + A)\vec{\varepsilon}\,^i$}\;
        \textcolor{green}{$a \leftarrow\frac{dKs_{\text{max}}\tau^2}{4\lambda}$}\Comment*[r]{\small Eq.~\eqref{eq:efficiency-lower-bound}}
        \textcolor{green}{$b \leftarrow (A - \frac{\sigma^2(\lambda+d+1)Ks_{\text{max}}}{4\lambda})\norm{\vec{g}}^2_2 - \frac{A^2dKs_{\text{max}}}{4\lambda}$}\;
        \textcolor{green}{$M \leftarrow (1-\beta) M + 2a/b$} \Comment*[r]{\small Eq.~\eqref{eq:optimal-M}}
        \tcp{\small See~\cite{Hansen16a}}
        Update $\mathbf{C}$ and $\sigma$ with \textcolor{green}{$\{w_i\}_i$} and $\{\vec{\varepsilon}\,^i\}_i$\;
    }
    \textbf{Output:} $\vec{m}$
\end{algorithm}

\paragraph{CMA-ES}
The Covariance Matrix Adaptation Evolution Strategy (CMA-ES)~\cite{Hansen16a} is a widely used black-box optimization algorithm for continuous, single-objective problems~\cite{hansen2008method, LoshchilovH16, SalimansHCS17, bonet2023performance}.
CMA-ES maintains a ``center of mass'' $\vec{m}\in\mathbb{R}^d$, which estimates the global minimum. 
In each iteration CMA-ES generates independent and identically distributed (i.i.d.) candidate solutions $\{\vx^{\,i}\}_{i=1}^\lambda$ from a multivariate Gaussian:

\begin{equation}\label{eq:default-update-of-x}
    \vx^{\,i} = \vec{m} + \mathbf{C}^{1/2}\vec{\varepsilon}^{\,i}, \; \vec{\varepsilon}^{\,i}\sim \sigma\N(0, \mathbf{I}),\; i\in[1..\lambda],
\end{equation}
where $\vec{\varepsilon}^{\,i}$ is referred to as the $i$-th mutation vector, $\sigma$ is the step size that scales the mutation vector, and $\mathbf{C}$ is the covariance matrix.
Both $\sigma$ and $\mathbf{C}$ are self-adapted within CMA-ES~\cite{Hansen16a}.
CMA-ES ranks the candidates based on their objective values (with ties broken randomly) as follows: $\loss(\vx^{\,1:\lambda}) < \loss(\vx^{\,2:\lambda}) < \ldots < \loss(\vx^{\,\lambda:\lambda})$. 

The center of mass is then updated using weighted recombination of the top-$\mu$ candidates ($\mu < \lambda$):
\begin{equation} \label{eq:default-update-of-m}
    \vec{m} \leftarrow \vec{m} + \vz, \; \vz=\sum_{i=1}^\mu w_i\mathbf{C}^{1/2}\vec{\varepsilon}\,^{i:\lambda}, \;
    \sum_{i=1}^{\mu} w_i = 1.
\end{equation}
Here, $\vec{\varepsilon}^{\,i:\lambda}$ is the mutation vector that generates $\vx^{\,i:\lambda}$.
By default, CMA-ES uses a monotonically decreasing function w.r.t.~the ranking of these candidates for assigning the weight $w_i$.

For a noisy objective $\mloss$, it is common to re-evaluate each candidate point $\vx$ over $M$ trials and provide CMA-ES with the average function value $\mloss$.
Based on CLT, we approximately have $\mloss(\vx)\sim \loss(\vx) + \mathcal{N}(0,\tau^2/M)$, when $M$ is large.
There are several proposals to extend CMA-ES to minimize noisy functions (see Sec.~\ref{sec:related-work}).

\section{Adaptive Re-evaluation (AR-CMA-ES)}
\label{sec:method}
We summarize our method in Alg.~\ref{alg:cma-es}, where our modifications to the standard CMA-ES are highlighted.
Our method extends the CMA-ES algorithm to handle additive noise, with the primary objective of dynamically estimating the optimal number of function re-evaluations required for each candidate. 
We first modify Eq.~\eqref{eq:default-update-of-m} to consider all mutation vectors:
\begin{equation}\label{eq:update-CMA}
    \vec{z} = \sum_{i=1}^\lambda w_i\mathbf{C}^{1/2}\vec{\varepsilon}^{\,i}, \quad\vepsilon^{\,i} \sim\sigma\mathcal{N}(0, \mathbf{I}),
\end{equation}
where the recombination weights are determined from the noisy function values.
The rationale for this consideration is that taking an average over a larger set helps reduce the impact of noises on the weight.

\paragraph{Modify the search direction $\vec{z}$:}
Instead of using the default weighting scheme, we consider the proportional weights for ease of analysis, a method commonly applied in evolutionary algorithms~\cite{EmmerichS018}.
This approach assigns a positive weight proportional to the loss value of each mutation
\begin{equation} \label{eq:proportional-weight}
    w_i = \frac{\Delta\mloss^i + A}{\sum_{k=1}^\lambda\Delta\mloss^k + \lambda A}, 
\end{equation}
where $\Delta\mloss^i= \mloss(\vec{m}) - \mloss(\vec{m}+ \mathbf{C}^{1/2}\vepsilon^{\, i})$ represents the change in the noisy objective function value, and $A$ is chosen as the smallest possible value that ensures all weights remain positive with high probability.
Considering the first-order Taylor expansion of $\mloss(\vec{m}+ \mathbf{C}^{1/2}\vepsilon^{\, i})$, then,
\begin{align}
    &\Delta\mloss^i = -\left\langle \nabla \loss(\vec{m}),\mathbf{C}^{1/2} \vepsilon^{\,i}\right\rangle + \mathcal{O}\left(\big\|\mathbf{C}^{1/2} \vepsilon^{\, i}\big\|^2_2\right) + \delta^{\,i} \nonumber\\
    &= -\langle \vec{g},\vepsilon^{\,i} \rangle + R \norm{\vepsilon^{\, i}}^2_2 + \delta^{\,i},  \;\text{for some } R \in \mathbb{R},
\end{align}
with $\vec{g} = \mathbf{C}^{1/2}\nabla\loss(\vec{m})$ and $\delta^{i} \sim\mathcal{N}\left(0, \tau^2/M\right)$ being i.i.d. noise in function value, and independent of $\vepsilon\,^i$.

When the step-size $\sigma$ is small, we have $\Delta\mloss^i \sim \sigma\norm{\vec{g}}_2 + \mathcal{N}(0, \tau^2/M)$.
Choosing $A \geq c\tau/\sqrt{M} - \sigma\norm{\vec{g}}_2$ will ensure $\Pr(\Delta\mloss^i + A \leq 0) \leq \Phi(-c)$ (e.g., $c=3$ gives ca. $0.15\%$ chance of realizing negative weights).
Also, since $A$ is a probabilistic upper bound of $\Delta\mloss^i$, we can relax the denominator of Eq.~\eqref{eq:proportional-weight} to $2\lambda A$, which leads to a modified search direction:
\begin{equation} \label{eq:modified-search-direction}
    \vec{z}\,' =  \frac{1}{2\lambda A}\sum_{i=1}^\lambda (\Delta\mloss^i+A)\mathbf{C}^{1/2}\vec{\varepsilon}^{\,i}.
\end{equation}
The modified search direction $\vec{z}\,'$ is easier to analyze and keeps the direction of $\vec{z}$ with high probability:
\begin{equation*} 
    \vec{z}\,' =  \frac{\sum_{k=1}^\lambda\Delta\mloss^k + \lambda A}{2\lambda A}\vec{z} = \left(\frac{1}{2} + \frac{1}{2c}\mathcal{N}(0,1)\right)\vec{z}.
\end{equation*}
The probability that $\vec{z}\,'$ inverts $\vec{z}$ is $1-\Phi(c)$ which is negligible for $c\geq3$. Also, notice that $\E(\norm{\vec{z}'} \mid \vec{z} ) = \frac{1}{2}\norm{z}$. It suffices to halve CMA-ES's parameter in the step-size adaptation to ensure our modification does not affect other dynamics thereof. Hence, we can safely use the modified search direction $\vec{z}\,'$ in the following analysis.

\paragraph{Efficiency in noisy optimization}
For a search algorithm, it is natural to maximize the expected improvement induced by the random search direction $\vec{z}\,'$, i.e., $\E(\loss(\vec{m}) - \loss(\vec{m}+ \vec{z}\,'))$. Since $\vec{z}\,'$ is determined from the noisy function values, the more function re-evaluations ($M$) we use, the more likely $\vec{z}\,'$ would be a descending direction. Hence, in the noisy scenario, it is more sensible to maximize the expected improvement while minimizing $M$, resulting in an \emph{efficiency metric} (similar to the one proposed in~\cite{gu2021adaptive})
\begin{equation}\label{eq:shot-efficiency}
    \gamma = \frac{\E\left[\loss(\vec{m}) - \loss(\vec{m}+ \vec{z}\,')\right]}{M}.
\end{equation}
Given an arbitrary black-box function, it is challenging to compute the exact form $\gamma$. Instead, we seek a lower bound of it and then determine the optimal value of $M$ by maximizing the lower bound.

Firstly, we consider a change of basis of $\mathbb{R}^d$, i.e., $\forall i\in[1..d], \vec{e}_i^{\,\prime} = \mathbf{C}^{1/2}\vec{e}_i$. Note that $\Delta\mloss$ is not affected by the change of basis. In the new coordinate system, the search direction is:
\begin{align}\label{eq:basis-change}
    &\vec{v}\,'= \mathbf{C}^{-1/2}\vec{z}\,' = \frac{1}{2\lambda A}\sum_{i=1}^\lambda \underbrace{(\Delta\mloss^i + A) \vepsilon^{\,i}}_{\vec{v}^{\,i}}
\end{align}
We can obtain the first moment and the second moment of the component of $\vec{v}^{\,i}$  (see Appendix~\ref{sec:indiv_components} for the details).
For $k\in[1..d]$ we have:
\begin{align}
    \E\left[v^i_k\right] &= -g_k\sigma^2 \label{eq:first-moment} \\
    \!\!\!\!\E\left[(v^i_k)^2\right] &= \frac{\tau^2\sigma^2}{M} + \left(\norm{\vec{g}}^2_2 + 2g_k^2\right)\sigma^4 + A^2\sigma^2 \label{eq:second-moment}
\end{align}
where $v^i_k$ and $g_k$ are the $k$-th component of $\vec{v}^{\,i}$ and $\vec{g}$, respectively.

Using the above statistical property of $\vec{v}\,'$ and quadratic upper bound of the loss function (see Thm.~\ref{th:QUB}), we bound from below the expected improvement (see Appendix \ref{sec:lower-bound-improvement} for the derivation):
\begin{align}
    &\E\left[\loss(\vec{m}) - \loss(\vec{m} + \vec{z}\,') \right] \nonumber\\
   	& = \E\left[\loss(\vec{m}) - \loss(\vec{m} + \mathbf{C}^{1/2}\vec{v}\,') \right] \label{eq:EVG1} \\
    &\geq \E\left[-\left\langle \vec{g}, \vec{v}\,'\right\rangle - \frac{K}{2}\norm{\mathbf{C}^{1/2}\vec{v}\,'}^2_2\right] \label{eq:EVG2}\\
    &\geq -\E\left\langle\vec{g}, \vec{v}\,'\right\rangle - \frac{Ks_{\text{max}}}{2}\E\norm{\vec{v}\,'}^2_2 \label{eq:EVG4}\\
    &= \frac{\sigma^2}{2A}\norm{\vec{g}}^2_2 - \frac{\sigma^4(\lambda+d+1)Ks_{\text{max}}}{8\lambda A^2} \norm{\vec{g}}^2_2 - \frac{dKs_{\text{max}}\sigma^2}{8\lambda } \nonumber\\
    &-\frac{1}{M}\frac{\sigma^2dKs_{\text{max}}\tau^2}{8\lambda A^2}  \label{eq:lower-bound-improvement}
\end{align}
where $s_{\text{max}}$ is the largest eigenvalue of $\mathbf{C}$ and $K$ is the Lipschitz constant of $\nabla\loss$.
\begin{remark}
    From Eq.~\eqref{eq:EVG1} to \eqref{eq:EVG2}, we use the quadratic upper bound of real-analytic functions (see Theorem \ref{th:QUB}). From Eq.~\eqref{eq:EVG2} to \eqref{eq:EVG4}, we take the fact that $\norm{\mathbf{C}^{1/2}\vec{v}\,'}_2 \leq \norm{\mathbf{C}^{1/2}}_2\norm{\vec{v}\,}_2$, and $\norm{\mathbf{C}^{1/2}}_2 = \sqrt{s_{\text{max}}}$. The derivation of Eq.~\eqref{eq:lower-bound-improvement} from \eqref{eq:EVG4} requires Eq.~\eqref{eq:first-moment} and Eq.~\eqref{eq:second-moment}.
\end{remark}
Consequently, we obtain a lower bound on the efficiency:
\begin{equation} \label{eq:efficiency-lower-bound}
	\gamma \ge \frac{\sigma^2}{2A^2} \left(a M^{-2} + b M^{-1}\right),
\end{equation}
where, 
\begin{align*}
	 &a = \frac{dKs_{\text{max}}\tau^2}{4\lambda },\\
	 &b = \left(A - \frac{\sigma^2(\lambda+d+1)Ks_{\text{max}}}{4\lambda}\right)\norm{\vec{g}}^2_2 - \frac{A^2dKs_{\text{max}}}{4\lambda}.
\end{align*}
Eq.~\eqref{eq:efficiency-lower-bound} is quadratic function of $M^{-1}$. Obviously, $a>0$. If term $b>0$, then there is a unique maximizer thereof in $[0, \infty)$: 
\begin{equation} \label{eq:optimal-M}
    M^{*} = -\frac{2a}{b}.
\end{equation}

\begin{remark}
    In practice, we notice that the optimal value calculated in Eq.~\eqref{eq:optimal-M} is prone to numerical instability. Therefore, we apply exponential smoothing to $M^*$ in each iteration:
    \begin{equation*}
         M_{t} = (1-\beta)M_{t-1} + \beta M^{*}, \; \beta \in (0,1).
    \end{equation*}
    The initial value $M_0$ should be small and specified by the user. We take $M_t$ re-evaluations for each candidate solution in iteration $t$. Whenever $b<0$, $M^*$ is negative, we simply ignore $M^*$, pause the above smoothing operation, and use the $M_t$ from the last iteration for the re-evaluation.
\end{remark}
We further validate the theoretical lower bound of $\gamma$ on the 10-dimensional noisy sphere function with $\tau=1, \lambda=20$: we measure, for a range of different re-evaluation number $M$, the empirical improvement over $M$ from 50 independent simulations of the mutation at iteration 100 (or any other iterations in the convergent phase). We show the result in Fig.~\ref{fig:empirical-lower-bound}, which numerically validates the correctness of the lower bound and, more importantly, shows that the lower bound curve resembles the trend of the empirical one. As a result, the optimal re-evaluations $M^*$ (green star) upper-bounds the optimum estimated from the empirical curve (red star). \\
There are a few unknown parameters needed in the lower bound (Eq.~\eqref{eq:efficiency-lower-bound}). We discuss how to estimate those as follows. 

\paragraph{Estimate the Lipschitz constant for $\nabla\loss$:}
Lipschitz constant estimation (Lipschitz learning algorithms) is an active research topic~\cite{GonzalezDHL16,StronginBB19,HuangRC23}, and we have no intention of developing new estimation methods in this work. Instead, we propose a feasible solution based on fitting a local Gaussian process regression to the population. For black-box problems, we employ a similar estimation method as in~\cite{GonzalezDHL16}: we fit a local Gaussian process model to the population $\{(\vx\,^i, \mloss(\vx\,^i))\}_{i=1}^\lambda$, specified by zero prior mean function and Gaussian kernel with white noise to handle the noisy function value: $k(\vx, \vx') = \exp(-\theta \norm{\vx - \vx'}^2) + \tau/\sqrt{M} \mathds{1}_{\{\vx'\}}(\vx)$. Let $H(\vx)$ be the Hessian matrix of the posterior mean function at point $\vx$ and $\mathcal{C}$ denote the convex hull of $\{\vx\,^i\}_i$, we can show $\widehat{K} = \max_{\vx \in \mathcal{C}} \norm{H(\vx)}_2$ is a valid estimate to Lipschitz constant of $\nabla\loss$ restricted to $\mathcal{C}$: Applying the mean-value theorem, we have 
$$
    \norm{\nabla\loss(\vx) - \nabla\loss(\vx')}_2 = \norm{H(\vec{z})(\vx - \vx')} \leq \norm{H(\vec{z})}_2 \norm{\vx - \vx'}_2,
$$ 
where $H(\vec{z})$ is the Hessian at $\vec{z}$ and $\vec{z} = (1-t)\vx + t\vx'$, for some $t \in (0, 1)$. We have, for all $\vx,\vx'\in\mathcal{C}, \norm{\nabla\loss(\vx) - \nabla\loss(\vx')}_2 \leq \max\limits_{\vz\in \mathcal{C}} \norm{H(\vz)}_2 \norm{\vx - \vx'}_2$, suggesting $\max\limits_{\vz\in \mathcal{C}} \norm{H(\vz)}_2$ can serve as the Lipschitz constant estimation.
To efficiently compute $\widehat{K}$, we approximately solve the above maximization problem by sampling $100d$ points u.a.r.~in $\mathcal{C}$.

\paragraph{Estimate the noise $\tau$:}
Since we assume homogeneous additive noise, it suffices to calculate the unbiased sample standard deviation $\hat{s}(M)$ of the function value at a randomly chosen point for various values of $M$ before invoking CMA-ES. Using the relationship $\E(\hat{s}(M)) = \tau / \sqrt{M}$, a simple curve-fitting of $\hat{s}(M)$ can provide a robust estimate for $\tau$.

\paragraph{Estimate $\vec{g}$:}
In Eq.~\eqref{eq:first-moment} implies that the mutation vectors are unbiased estimators of the gradient: $\vec{g} = -\E(\vec{v}^{\, i})/\sigma^2
$ for $i\in[1..\lambda]$. We further reduce the variance of this estimator by averaging over all candidates, i.e., $  \vec{g}^{\,*} = -\lambda^{-1}\sum_{i=1}^\lambda \vec{v}^{\, i} / \sigma^2$. Taking Eq.~\eqref{eq:second-moment}, we have the variance of the estimate: $\operatorname{Var}(g_k) = (2/\lambda - 1)g_k^2 + \norm{\vec{g}}^2/\lambda + (\tau^2/M + A^2)/\lambda\sigma^2, k\in[1..d]$. Hence, the variance is small either the population size is large or $\norm{\vec{g}}$ is small, which happens when CMA-ES approaches a local minimum ($\vec{g} = 0$).
For the sake of numerical stability, we exponentially smooth $\vec{g}^{\,*} $ values in the past: $\vec{g} \xleftarrow[]{} (1-\alpha)\vec{g} + \alpha \vec{g}^{\,*}, \alpha \in (0,1)$.

\begin{figure}[t]
    \centering
    \includegraphics[width=1.0\linewidth, trim=2mm 2mm 0mm 0mm,clip]{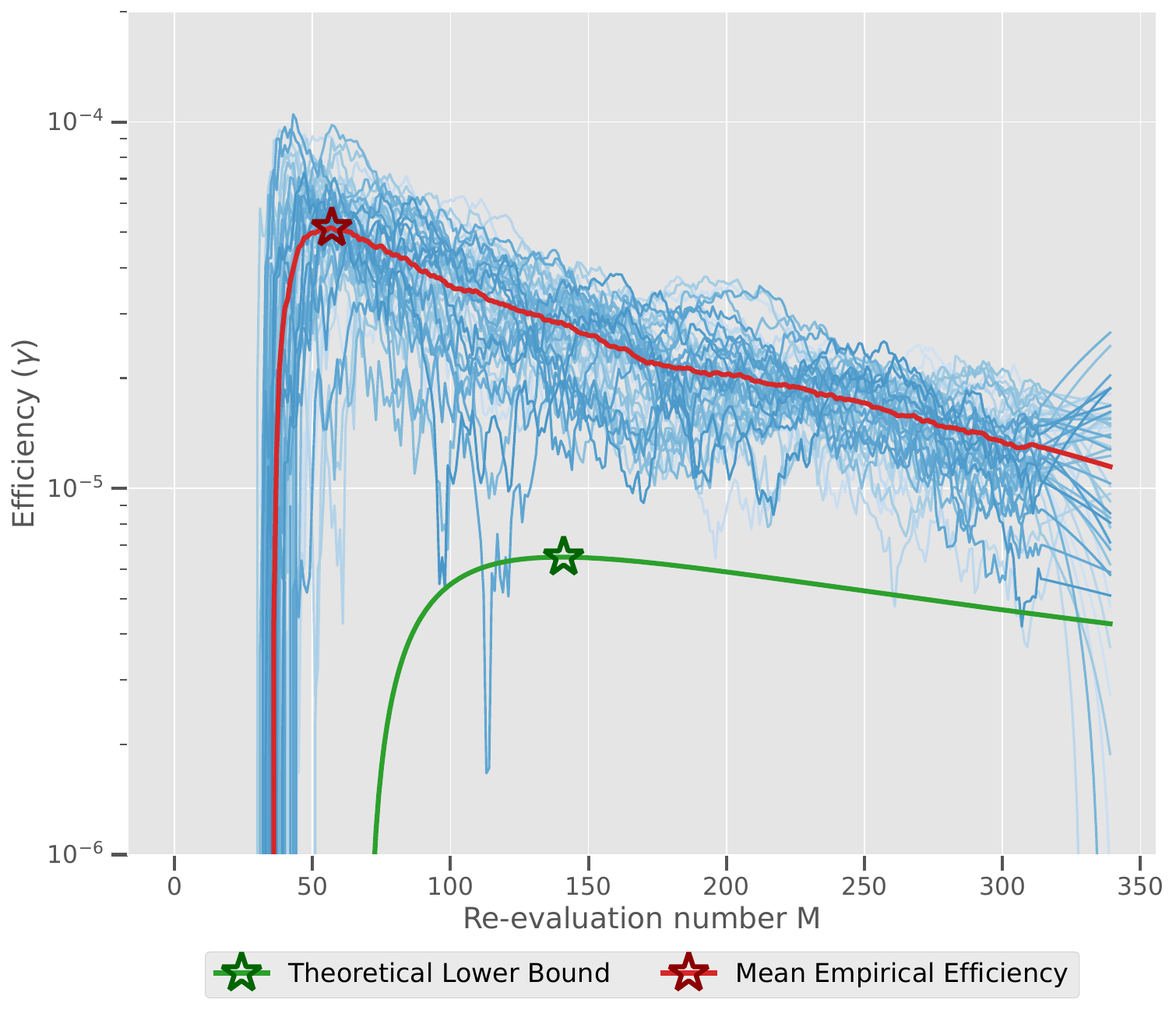}
    \caption{
    On the sphere function, the theoretical lower bound of the efficiency (green curve) and the empirical efficiency curve (red curve), i.e., the empirical improvement over the re-evaluation number, estimated from 50 independent simulations of the mutation of CMA-ES at iteration 100. Each simulated result is shown in light blue curves. We depict, in the star symbol, the maximum of both empirical and theoretical curves.}
    \label{fig:empirical-lower-bound}
\end{figure}
\paragraph{Time complexity:}
Our method incurs small time complexity in addition to the standard CMA-ES: Eq.~\eqref{eq:optimal-M} only involves a constant number of arithmetic operations; the largest eigenvalue $s_{\text{max}}$ of $\mathbf{C}$ is provided internally by the standard CMA-ES. It takes $\mathcal{O}(\lambda)$ time to estimate $\vec{g}$ and takes $\mathcal{O}(1)$ to estimate the noise level $\tau$ since the latter is only executed once. The Lipschitz estimation takes $\mathcal{O}(\lambda^3)$ time to fit the Gaussian process and $\mathcal{O}(\lambda d^2)$ to compute $\norm{H(\vx)}_1$ (the Hessian of the posterior mean function). Since, in practice, the population size is small - typically $\lambda\in\Theta(\log d)$, the actual CPU time used in Lipschitz estimation is marginal.

\renewcommand{\arraystretch}{1.2}
\begin{table*}[ht]
\small 
\caption{Numerical verification of AR-CMA-ES against CMA-ES for $d=20$ on noiseless test functions \(\tau^2 = 0\) and functions with low noise levels \(\tau^2 = 1\). 
For each algorithm, we report the mean and standard error of the final noiseless precision achieved over $20$ runs. 
Different evaluation budgets were used for noisy and noiseless settings--please refer to the Results section for further details.\label{table:ar-cma-to-cma}}
\centering

\begin{tabular}{|l|c|c|c|c|}
\hline
\textbf{Problem} & \textbf{AR-CMA-ES ($\tau^2 = 0$)} & \textbf{CMA-ES ($\tau^2 = 0$)} & \textbf{AR-CMA-ES ($\tau^2 = 1$)} & \textbf{CMA-ES ($\tau^2 = 1$)} \\ \hline
Sphere & $1.24 \times 10^{-5} \pm 6.86 \times 10^{-6}$ & $1.80 \times 10^{-6} \pm 1.28 \times 10^{-6}$ & $1.06 \times 10^{-7} \pm 5.63 \times 10^{-8}$ & $3.40 \times 10^{-5} \pm 2.01 \times 10^{-5}$ \\ \hline
Ellipsoid & $1.31 \times 10^{-3} \pm 8.09 \times 10^{-4}$ & $4.03 \times 10^{-4} \pm 3.21 \times 10^{-4}$ & $9.84 \times 10^{-8} \pm 5.85 \times 10^{-8}$ & $6.53 \times 10^{-6} \pm 3.63 \times 10^{-6}$ \\ \hline
HyperEllipsoid & $3.11 \times 10^{-4} \pm 2.52 \times 10^{-4}$ & $3.84 \times 10^{-5} \pm 3.21 \times 10^{-5}$ & $8.36 \times 10^{-8} \pm 3.56 \times 10^{-8}$ & $5.95 \times 10^{-6} \pm 2.44 \times 10^{-6}$ \\ \hline
RotatedEllipsoid & $1.60 \times 10^{-3} \pm 1.32 \times 10^{-3}$ & $5.40 \times 10^{-4} \pm 3.82 \times 10^{-4}$ & $9.21 \times 10^{-8} \pm 4.74 \times 10^{-8}$ & $6.63 \times 10^{-6} \pm 4.00 \times 10^{-6}$ \\ \hline
RotatedHyperEllipsoid & $3.27 \times 10^{-4} \pm 2.59 \times 10^{-4}$ & $5.63 \times 10^{-5} \pm 3.85 \times 10^{-5}$ & $8.75 \times 10^{-8}\pm 3.71 \times 10^{-8}$ & $5.42 \times 10^{-6} \pm 2.14 \times 10^{-6}$ \\ \hline
Rastingrin & $1.53 \times 10^{2} \pm 4.65 \times 10^{1}$ & $1.59 \times 10^{2} \pm 5.59 \times 10^{1}$ & $ 1.24 \pm 0.85 $ & $ 72.10 \pm 38.80$ \\ \hline
Trid & $1.46 \times 10^{-11} \pm 1.49 \times 10^{-11}$ & $1.46 \times 10^{-11} \pm 1.77 \times 10^{-11}$ & $1.46 \times 10^{-7} \pm 8.11 \times 10^{-8}$ & $2.74 \times 10^{-5} \pm 1.18 \times 10^{-5}$ \\ \hline
Bohachevsky & $2.29 \times 10^{-2} \pm 2.44 \times 10^{-2}$ & $3.89 \times 10^{-3} \pm 4.05 \times 10^{-3}$ & $2.30 \times 10^{-7} \pm 3.71 \times 10^{-7}$ & $6.15 \times 10^{-6} \pm 3.02 \times 10^{-6}$ \\ \hline
CosineMixture & $9.14 \times 10^{-7} \pm 1.27 \times 10^{-6}$ & $1.56 \times 10^{-7} \pm 1.49 \times 10^{-7}$ & $9.25 \times 10^{-8} \pm 4.15 \times 10^{-8}$ & $1.98 \times 10^{-5} \pm 5.82 \times 10^{-5}$ \\ \hline
Schwefel02 & $3.96 \times 10^{-3} \pm 2.95 \times 10^{-3}$ & $1.68 \times 10^{-3} \pm 1.41 \times 10^{-3}$ & $7.60 \times 10^{-8} \pm 2.70 \times 10^{-8}$ & $4.00 \times 10^{-3} \pm 1.79 \times 10^{-2}$ \\ \hline
\end{tabular}
\end{table*}

\begin{figure*}[!ht]
	\centering
        \includegraphics[width=0.85\linewidth]{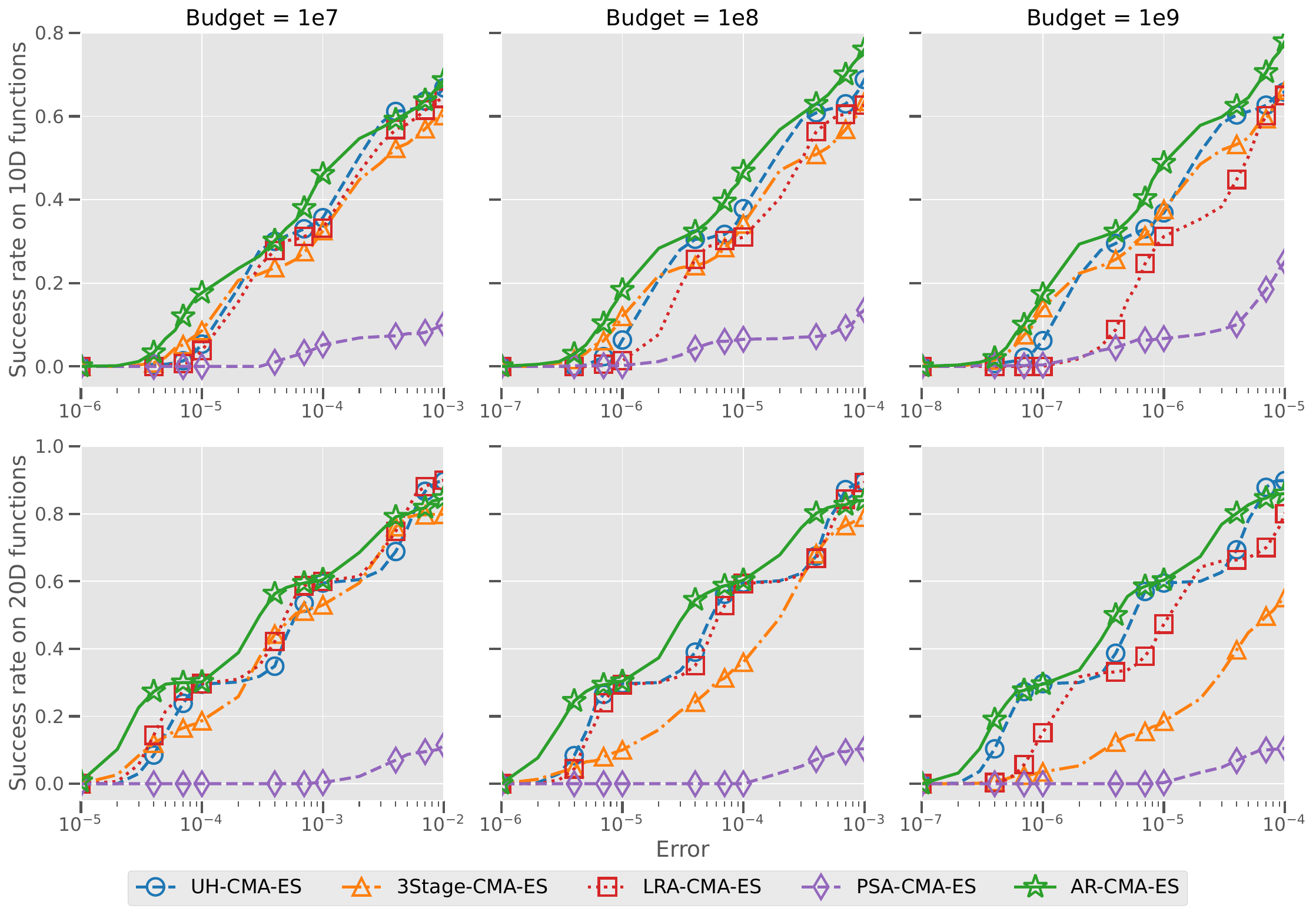}
        \caption{Empirical cumulative distribution functions (ECDFs) of the error ($\loss(\vec{m}) - \loss^*$) obtained from all 20 independent runs, three noise levels, and on all test functions. Top row: $d=10$; bottom: $d=20$. Three columns from left to right correspond to an evaluation budget of $10^7$, $10^8$, and $10^9$, respectively.}
	\label{fig:all10_20_ecdf}
\end{figure*}

\section{Related works} \label{sec:related-work}
\paragraph{Three-Stage CMA-ES:}
The authors in~\cite{PhysRevB.102.235122} propose a static schedule that divides the optimization process into three distinct stages, with the number of re-evaluations increasing ten-fold at each stage.
For example, with a budget of $10^7$ function re-evaluations, the method allocates $M_1 = 100, M_2 = 1\,000, M_3 = 10\,000$, and keeps a fixed ratio of 10:3:1 among the total function evaluations in three stages. Such a setup results in evaluations of approximately 7\,150, 2\,145, and 715 candidates at each stage, respectively. 
Despite its simplicity, this method has been shown to work well on quantum chemistry problems~\cite{PhysRevB.102.235122,bonet2023performance}. However, this method may not be as effective for other problems, as the fixed number of re-evaluations might either fall short or be excessive, potentially slowing down the convergence rate of CMA-ES.

\paragraph{Uncertainty handling CMA-ES:}
The Uncertainty handling CMA-ES(UH-CMA-ES) introduced in ref.~\cite{hansen2008method} presents an adaptive strategy that increases the re-evaluation number $M$ if significant ranking changes occur for some candidates when their noisy function values are recomputed with the current $M$.
Specifically, after evaluating each point in the population $\{\vx\,^i\}_i$ with $M$ re-evaluations, a random sub-population is selected to re-estimate the function values.
The entire population is then reordered based on these updated noisy values, and the ranking changes for each $\vx\,^i$ are compared before and after re-estimation.
UH-CMA-ES aggregates these rank changes across all candidates to determine whether $M$ should be adjusted. If the indicator is positive, $M$ is increased multiplicatively; otherwise, it stays the same.

\paragraph{Population Size Adaptation CMA-ES:}
The population size adaptation in CMA-ES builds on the Information Geometric Optimization (IGO) framework ~\cite{nishida2016population}, where population size is treated as the number of Monte Carlo samples used to estimate the natural gradient. This approach reduces gradient variability in multimodal and noisy functions compared to unimodal ones.

Applied to CMA-ES, population size adjusts based on the length of the evolution path ~\cite{NishidaA18}. A normalization factor ensures accurate parameter update assessment, while step-size adjustment maintains stability during population changes. When updates lack precision, population size increases to enhance exploration; when precision is sufficient, it decreases to focus on exploitation. This dynamic balance between exploration and exploitation improves performance in noisy and multimodal environments.

\paragraph{Learning Rate Adaptation CMA-ES:}
The so-called Learning Rate Adaptation CMA-Es (LRA-CMA-ES) presented in ref.~\cite{nomura2023cma} introduces a dynamic adjustment of the learning rates ($\eta^t_m$ and $\eta^t_{\mathbf{C}}$) on a per-iteration basis.
Effectively, such adaptation translates into tuning the updates $\Delta^t_m$  and $\Delta^t_{\mathbf{C}}$.
As such, the updating rules of the center of mass and the covariance matrix are $m^{t+1}=  m^t + \eta^t_m \Delta^t_m$  and $\mathbf{C}^{t+1}=  \mathbf{C}^t + \eta^t_{\mathbf{C}} \Delta^t_{\mathbf{C}}$.
It estimates the signal-to-noise ratio as the fraction between the expected value of the updating vector and its variance.
The adaptive learning rate mechanism seeks to maintain a constant signal-to-noise ratio (SNR) provided as a hyperparameter.
Thus, when the empirical SNR is higher than the provided constant, the learning rate is increased, and when it is lower, the learning rate is reduced.

\begin{table*}[ht]
\caption{Wall-clock time comparison for a budget of $10^7$ re-evaluations on $d=20$, averaged over 20 runs (unit: secs) using an Intel Core i7-9750H processor (6 CPU cores) and 32GB RAM.}
\label{tab:performance_comparison}
\centering
\begin{tabular}{|l|c|c|c|c|c|c|}
\hline
\textbf{Problem} & \textbf{AR-CMA-ES} & \textbf{CMA-ES} & \textbf{3-Stage-CMA-ES} & \textbf{UH-CMA-ES} & \textbf{LRA-CMA-ES} & \textbf{PSA-CMA-ES} \\ \hline
Sphere & $16.2 \pm 0.44$ & $22.7 \pm 0.80$ & $17.8 \pm 0.83$ & $19.6 \pm 0.81$ & $24.6 \pm 2.78$ & $25.4 \pm 3.65$ \\ \hline
Ellipsoid & $16.1 \pm 0.31$ & $29.4 \pm 3.76$ & $18.5 \pm 1.85$ & $21.1 \pm 1.69$ & $41.7 \pm 2.08$ & $36.9 \pm 4.66$ \\ \hline
HyperEllipsoid & $16.0 \pm 0.00$ & $28.4 \pm 4.52$ & $17.9 \pm 0.45$ & $20.2 \pm 0.95$ & $34.9 \pm 2.10$ & $29.8 \pm 5.60$ \\ \hline
RotatedEllipsoid & $16.0 \pm 0.00$ & $37.0 \pm 4.37$ & $18.2 \pm 0.55$ & $21.1 \pm 1.07$ & $46.1 \pm 4.38$ & $38.2 \pm 4.30$ \\ \hline
RotatedHyperEllipsoid & $16.0 \pm 0.00$ & $33.0 \pm 1.57$ & $17.4 \pm 0.49$ & $20.2 \pm 0.89$ & $34.5 \pm 3.30$ & $29.1 \pm 2.55$ \\ \hline
Rastingin & $16.6 \pm 0.69$ & $32.9 \pm 3.30$ & $18.1 \pm 1.43$ & $20.8 \pm 1.67$ & $34.9 \pm 3.48$ & $32.7 \pm 3.25$ \\ \hline
Trid & $18.0 \pm 0.32$ & $23.5 \pm 0.51$ & $18.0 \pm 0.56$ & $23.2 \pm 1.07$ & $34.6 \pm 1.05$ & $27.6 \pm 0.82$ \\ \hline
CosineMixture & $15.8 \pm 0.37$ & $30.4 \pm 3.60$ & $18.1 \pm 0.72$ & $17.1 \pm 0.51$ & $37.0 \pm 4.28$ & $32.3 \pm 9.40$ \\ \hline
Bohachevsky & $17.4 \pm 0.51$ & $38.4 \pm 1.14$ & $19.1 \pm 0.45$ & $22.7 \pm 2.34$ & $52.8 \pm 3.35$ & $45.2 \pm 7.58$ \\ \hline
Schwefel & $16.5 \pm 0.51$ & $26.9 \pm 3.25$ & $17.6 \pm 0.60$ & $20.3 \pm 2.18$ & $33.4 \pm 1.35$ & $26.9 \pm 3.98$ \\ \hline
\end{tabular}
\end{table*}

\section{Experiments} \label{sec:experiments}
\paragraph{Experiments setup:}
We make an empirical comparison of AR-CMA-ES against the most advanced methods: UH-CMA-ES, Three-Stage CMA-ES, PSA-CMA-ES, and LRA-CMA-ES. 
We thoroughly re-implement them by integrating their original source code with the modular CMA-ES~\cite{de2021tuning} framework, also considering the details in the original publication to the best of our ability\footnote{The source code can be accessed at \\
\url{https://anonymous.4open.science/r/ShotFrugal-7CD4}
}.

For the objective functions, we choose ten standard artificial test functions (see Table~\ref{tab:functions} in the Appendix \ref{sec:all_experiments} for their definition).
These test functions encompass a wide range of landscapes, such as unimodal/multi-modal landscapes and dimension-separable and non-separable properties, which are considered difficult for numerical optimization.
To gather statistically relevant data, we will execute 20 independent runs for each test function.
Additionally, we add artificial noise in three levels: $\tau^2 \in \{1, 10, 100\}$. 
To make the comparison as fair as possible, we use the same population size of CMA-ES, $\mu = 50,\lambda = 100$, for all methods; the initial step size is set to $\sigma_0 = 0.1\times \norm{\vx_U - \vx_L}_{\infty}$, where $[\vx_L, \vx_U]\subset \R^d$ is the search space (see Table~\ref{tab:functions} in the appendix for the search space of each function).
For the methods we compare, we leave their remaining hyperparameter settings unchanged from the original publication.

To determine the coefficients $\alpha$ and $\beta$ used in exponential smoothing for our method, we extensively test various combinations of them, which results in setting $\alpha=0.1$ and $\beta=0.1$.
For the value of $A$ in Eq.~\eqref{eq:proportional-weight}, we choose the smallest measured $\Delta\mloss$ value among all candidates in each iteration.

Finally, we test all methods with different budgets of function evaluations, where we recap the re-evaluation number per candidate at $1\%$ of the total budget.

\paragraph{Results:}
First, since we modified CMA-ES's default recombination weights in Eq.~\ref{eq:proportional-weight}, we verify the performance of AR-CMA-ES against CMA-ES in the noiseless ($\tau = 0$) and noisy scenarios ($\tau = 1$). In Table~\ref{table:ar-cma-to-cma}, we see that in the noiseless cases, AR-CMA-ES is quite comparable to CMA-ES while under the additive noise ($\tau = 1$), AR-CMA-ES significantly improves CMA-ES. For the noise-free case, we used a fixed budget of $10^4$ function re-evaluations for each function for both AR-CMA-ES and CMA-ES. An exception was made for the Trid function, which required a larger budget of $10^5$ re-evaluations to be successfully solved. For the noisy case, we allocated a significantly larger budget of $10^9$ function re-evaluations. To ensure a fair comparison, we also allowed standard CMA-ES to re-evaluate each candidate for $10^4$, as it will eventually reach a plateau for any input given candidate no of re-evaluations.

Second, we record the trajectory of the center of mass $\vec{m}$ and compute the corresponding noiseless function values $\loss(\vec{m})$.
Then, we compute the empirical cumulative distribution function (ECDF) of the optimization error $\loss(\vec{m}) - \loss^*$ upon the termination of each method ($\loss^*$ denotes the global optimal) for each combination of $d\in\{10, 20\}$ and evaluation budget in $\{10^7, 10^8, 10^9\}$.
Formally, ECDF of an algorithm is defined as $\operatorname{ECDF}(x) = \sum_{i=1}^{N}\mathds{1}_{[e_i, \infty)}(x) / N$, where $e_i$ is the optimization error observed in the $i$-th run.
We show the main ECDF curves in Fig.~\ref{fig:all10_20_ecdf}, which aggregates over all functions and noise levels.
Also, we included, in the appendix, the ECDFs on each function and noise level (Fig.~\ref{fig:all_functions_general_plots_10D} and~\ref{fig:all_functions_general_plots_20D}).

As we increase the budget and function dimension, and hence the hardness of the optimization task, AR-CMA-ES shows a substantial performance improvement compared to all other methods.
Particularly for relatively higher dimensions ($d=20$), we pointed out that the major benefit of our method lies in increasing the probability of hitting difficult error values quite a bit. 
As an example, with $d=20$ and a budget of $10^7$ function evaluations, AR-CMA-ES can reach an optimization error $\leq 4\times 10^{-5}$ with approximately \(27\%\) probability.
In contrast, for all other methods, the probability drops drastically, UH-CMA-ES: 9\%, Three-Stage-CMA-ES: 12\%, LRA-CMA-ES: 14\%, and PSA-CMA-ES: 0\%.
With a higher budget of $10^9$ function evaluations and \(d=20\), we observe a similar result; as such, our method found around 19\% of solutions with an optimization error $\leq 4\times 10^{-7}$, while  UH-CMA-ES achieved only 10\% and the other methods failed to achieve such threshold. However, we can observe two convergence points for all three budgets where several methods achieve a similar probability of success. With a budget $10^9$ and \(d=20\), we observe that AR-CMA-ES and UH-CMA-ES achieve so probability of success at a precision of $10^{-6}$ (around 30\%) and at a precision of $10^{-5}$ (around 60\%). However, our method still shows a significantly higher cumulative probability at almost all error values.
To see the effect of the noise level on the performance, we show in Fig.~\ref{fig:all_10D_noise} (in Appendix~\ref{sec:all_experiments}) the ECDF curves for each combination of dimensions, budgets, and noise levels. 
As the noise level increases, performance slightly decreases.
This behavior is due to overestimation of $M^*$, as the number of function re-evaluations is linearly dependent on the noise. 

For closer analysis, we showcase the ECDF and empirical convergence curve on the Trid function, which is a non-separable function across dimensions, making it a challenging problem for algorithms.
\begin{figure*}[!t]
	\centering
        \includegraphics[width=0.85\linewidth]{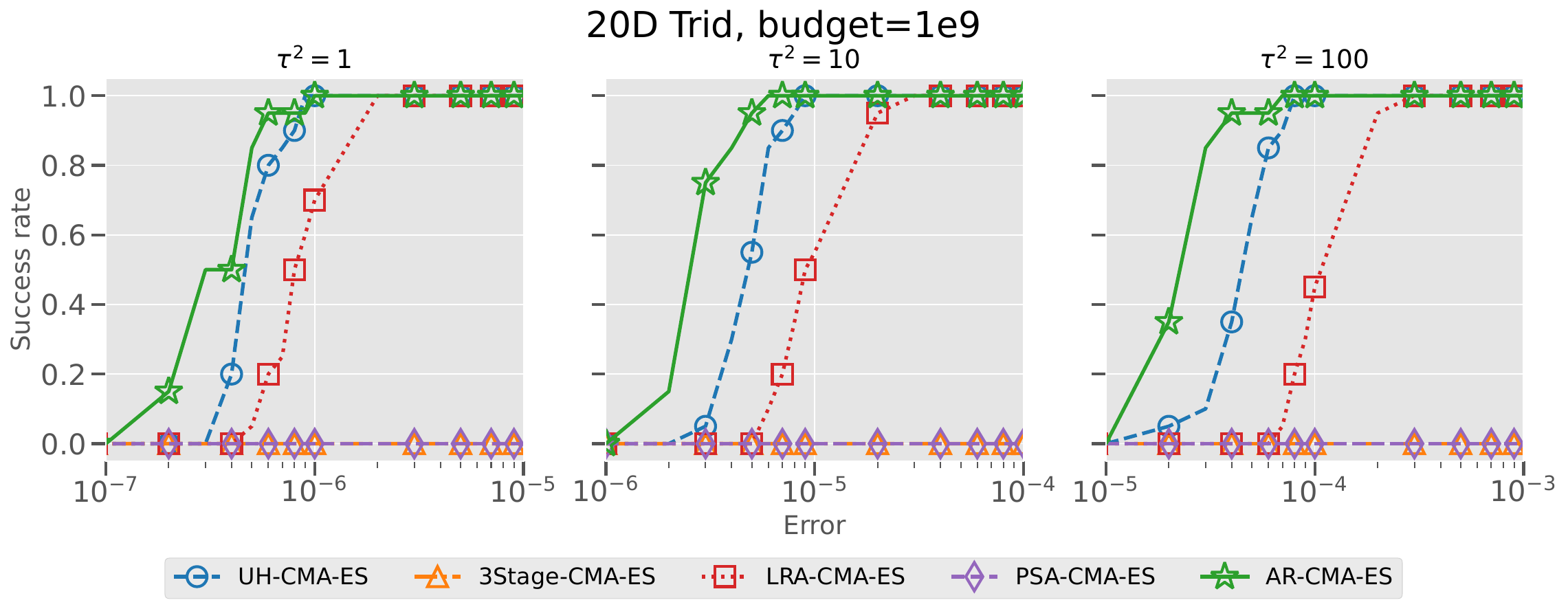}
        \includegraphics[width=0.85\linewidth]{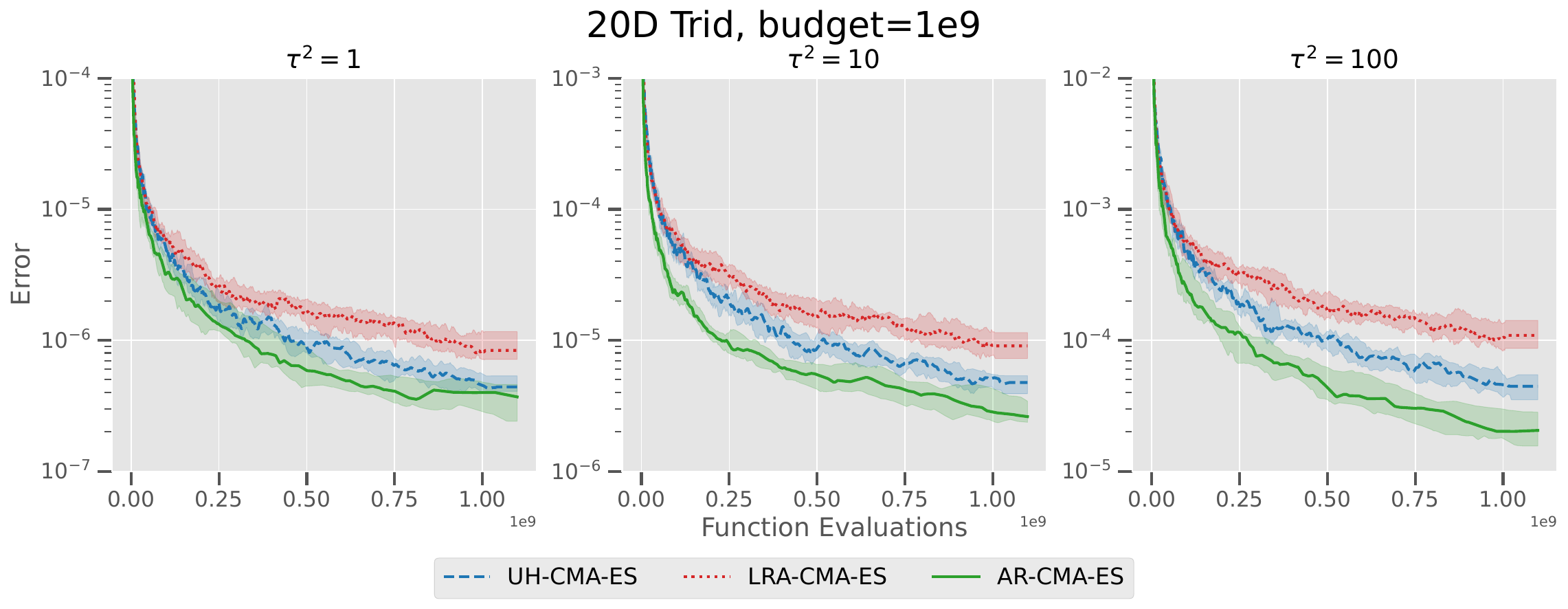}
        \caption{Top: On 20-dimensional Trid function (unimodal and non-separable), the empirical cumulative distribution function (ECDF) of the error ($\loss(\vec{m}) - \loss^*$) obtained with $10^9$ function evaluation budget for three different noise levels ($\tau^2 \in \{1, 10, 100\}$) separately. Bottom: Mean convergence curve - $\log_{10}(\loss(\vec{m}) - \loss^*)$ as a function of function evaluations. We see AR-CMA-ES outperforms other methods, and it is more advantageous when the noise $\tau$ gets larger. We observe that both Three-Stage and PSA-CMA-ES completely failed on this function.}
	\label{fig:trid_20D_ecdf}
\end{figure*}

Fig.~\ref{fig:trid_20D_ecdf} (top) shows the ECDF on a 20-dimensional Trid function with a budget of $10^9$ function evaluations and different noise levels ($\tau^2 \in \{1,10,100\}$). As discussed, the Trid function is non-separable across dimensions (the minimum cannot be found by searching along each dimension separately). We observe that AR-CMA-ES achieves substantial improvement compared to all other methods, while Three-Stage and PSA-CMA-ES failed to hit any small error value, indicated by their flat ECDF curve. In Fig.~\ref{fig:trid_20D_ecdf} (bottom), we draw the convergence curves - $\log_{10}(\loss(\vec{m}) - \loss^*)$ as a function of function evaluations. We see that AR-CMA-ES delivers a significantly steeper convergence than UH- and LRA-CMA-ES.

\paragraph{Performance on non-Lipschitz continuous functions} We also perform experiments with AR-CMA-ES on four ten-dimensional benchmark functions that are not Lipschitz continuous (see~\cref{tab:non-smooth-functions}) to assess how its performance is affected when its core theoretical assumption does not hold.
First, we compare the performance of AR-CMA-ES to CMA-ES in a noiseless setting, as shown in~\cref{tab:non-smooth-noiseless}. 
The results indicate that AR-CMA-ES performs comparably to CMA-ES, with no statistically significant difference on most problems--except for the  Sum Absolute function, where its performance is off by an order of magnitude.
Next, we evaluate all relevant algorithms under a moderate noise level of $\tau^2 = 1$.
As summarized in~\cref{tab:non-smooth-functions}, AR-CMA-ES consistently outperforms CMA-ES on all but one benchmark--the non-smooth version of the Griewank function--where its performance drops by two orders of magnitude.
Upon further analysis, we found this discrepancy to stem from inaccurate estimation of the Lipschitz constant, which is, in fact, ill-defined for this function.
Compared to other algorithms, AR-CMA-ES ranks second on the Sum Absolute function (slightly behind UH-CMA-ES) and shows statistically indifferent results on the Nesterov F1 and F2 benchmarks.
These results suggest that while AR-CMA-ES is designed with Lipschitz continuity in mind, it remains competitive even for non-smooth benchmark functions. 
Please refer to Appendix~\ref{sec:nonsmoothresults} for more details.


\section{Conclusion}
In this paper, we propose AR-CMA-ES, a novel noise-handling method for the famous CMA-ES algorithm under additive Gaussian white noise. We consider the expected improvement of the noiseless function value in one iteration of CMA-ES and derive a lower bound on it, provided the noise level and the Lipschitz constant of the function's gradient. Normalizing the lower bound by the re-evaluation number gives us an efficiency metric. Solving for the maximum efficiency, we obtain a simple expression of the optimal re-evaluation number.

This adaptive strategy enhances CMA-ES's performance by efficiently allocating function (re)-evaluation without significant computational overheads. AR-CMA-ES substantially outperforms several state-of-the-art noise-handling methods for CMA-ES and demonstrates a consistent advantage across different test functions, search dimensions, and noise levels. While AR-CMA-ES demonstrates significant improvements in handling additive noise, it exhibits the following limitations:
\begin{itemize}
    \item \textbf{Assumptions on noise characteristics}: AR-CMA-ES is designed with a focus on additive noise. If the noise characteristics deviate from this assumption, such as multiplicative noise or other forms of complex noise patterns, the derived expression might not hold any longer. Further research is needed to extend the method to handle a broader range of noise types effectively.
    \item \textbf{Impact of noise level}: The number of function re-evaluations in AR-CMA-ES is linearly dependent on the noise level $\tau$. As the noise level increases, this dependency can lead to a huge re-evaluation number, which might not be the best choice in high-noise environments.
    \item \textbf{Limited empirical validation}: While AR-CMA-ES demonstrates performance benefits on artificial test functions, its effectiveness on real-world problems remains to be fully explored. The empirical validation primarily focuses on synthetic functions that adhere to the assumptions about the function and noise type. Further experimentation is needed to evaluate the method's performance on functions that naturally conform to these assumptions. Examples include quantum loss functions, which are prevalent in quantum computing optimization tasks. Extending the empirical validation to encompass a broader range of real-world problems will provide deeper insights into the method's applicability and effectiveness in practical scenarios.
\end{itemize}
For future works, we will focus on addressing the above limitations and testing them on real-world optimization problems.

\section*{Acknowledgments}
The authors acknowledge support from all members of the applied Quantum algorithms (aQa) group at Leiden University.
YJP is supported by the `Quantum Inspire--the Dutch Quantum Computer in the Cloud' project [NWA.1292.19.194] of the NWA research program 'Research on Routes by Consortia (ORC)', funded by Netherlands Organization for Scientific Research (NWO).

\bibliographystyle{ACM-Reference-Format}
\bibliography{gecco-ref}

\clearpage
\appendix

\section{Quadratic Upper Bound}\label{sec:qub}
\begin{theorem}[Quadratic Upper Bound]\label{th:QUB}
    Assume a real-valued function $\loss\colon \R^d \rightarrow \R$ with Lipschitz continuous gradient, i.e., $\norm{\nabla\loss(\vx) - \nabla\loss(\vx')} \leq K\norm{\vx - \vx'}$ for all $\vx,\vx'\in\R^d$. The following upper bound holds: $\forall\vx, \vec{y}\in\R^d$,
    \begin{equation*}
         \loss(\vec{y}) \leq \loss(\vx) + \left\langle\nabla \loss(\vx), \vec{y}-\vx \right\rangle + \frac{K}{2}\norm{\vec{y}-\vx}^2_2.
    \end{equation*}
\end{theorem} 
\begin{proof}
    Let $\vec{p} = \vec{y} - \vec{x}$.
    By the Taylor theorem, we have:
    \begin{align*}
        &\loss(\vec{y}) - \loss(\vec{x}) = \int_0^1 \left\langle \nabla\loss(\vec{x} + t\vec{p}), \vec{p} \right\rangle  \ud t \\
        &= \int_0^1 \langle \nabla\loss(\vec{x} + t\vec{p}) - \nabla\loss(\vec{x}) , \vec{p}\rangle  \ud t  + \langle \nabla \loss(\vec{x}), \vec{p}\rangle \\
        &\leq \int_0^1\norm{\nabla\loss(\vec{x} + t\vec{p}) - \nabla\loss(\vec{x})}_2 \norm{\vec{p}}_2 \ud t + \langle \nabla f(\vec{x}), \vec{p}\rangle \\
        &\leq \norm{\vec{p}}_2 \int_0^1 K\norm{t\vec{p}}_2 \ud t +  \langle \nabla f(\vec{x}), \vec{p}\rangle \\
        &=\frac{K}{2} \norm{\vec{p}}_2^2 + \langle \nabla f(\vec{x}), \vec{p}\rangle
    \end{align*}
\end{proof}
Applying the above theorem to Eq.~\eqref{eq:basis-change}, we have:
\begin{align}
    &\loss(\vec{m} + \mathbf{C}^{1/2}\vec{v}\,') \nonumber\\
    &\leq \loss(\vec{m}) + \langle\nabla \loss(\vec{m}),\mathbf{C}^{1/2}\vec{v}\,'\rangle + \frac{K}{2}\norm{\mathbf{C}^{1/2}\vec{v}\,'}^2_2 \nonumber \\
    &\loss(\vec{m}) -\loss(\vec{m} + \mathbf{C}^{1/2}\vec{v}\,') \nonumber\\
    &\geq - \langle\vec{g},\vec{v}\,'\rangle 
    - \frac{K}{2}\norm{\mathbf{C}^{1/2}\vec{v}\,'}^2_2, \label{eq:lower_bound}
\end{align}
where $\vec{g} = \mathbf{C}^{1/2}\nabla \loss(\vec{m})$.

\section{Statistical moments of $\vec{v}^{\,i}$} \label{sec:indiv_components}
Assuming $\vec{v}^{\,i} = (\Delta\mloss^i + A)\vepsilon^{\,i}$, the individual component of it can be expressed as: for $k \in[1..d]$,
\begin{align}
    v^{i}_{k} &= (\Delta\mloss^i+A)\varepsilon^i_k \nonumber\\
              &= \left[ -\left\langle \vec{g}, \vepsilon^{\,i}\right\rangle + \delta^{i} + R \norm{\vepsilon^{\,i}}^2_2 + A\right]\varepsilon^{i}_k
\end{align}
where $\vec{\varepsilon}\,^1, \ldots, \vec{\varepsilon}\,^i, \ldots, \vec{\varepsilon}\,^\lambda \sim \sigma\mathcal{N}(0,\mathbf{I})$ are i.i.d., $\delta^i \sim \mathcal{N}(0,\tau^2/M)$, $R \in \mathbb{R}$, and $\delta^i$ is independent of $\{\vec{\varepsilon}\,^i\}_i$.

\paragraph{Proof of~\cref{eq:first-moment}} The first moment of each individual component is given by:
\begin{align}
    &\E\left[v^i_k\right] \nonumber\\
    &=\E\left[-\langle \vec{g},\vepsilon^{\,i}\rangle\varepsilon^i_k + \delta^i\varepsilon^i_k + R\norm{\vepsilon^{\,i}}^2_2\varepsilon^i_k + A \varepsilon^i_k\right] \nonumber\\
    &=-\underbrace{\E\left[\langle \vec{g},\vepsilon^{\,i}\rangle\varepsilon^i_k\right]}_{A_1} 
    + \underbrace{\E\left[\delta^i\varepsilon^i_k\right]}_{A_2}+R\underbrace{\E\left[\norm{\vepsilon^{\,i}}^2_2\varepsilon^i_k\right]}_{A_3} \nonumber\\
    &+A\underbrace{\E\left[\varepsilon^i_k\right]}_{A_4=0}\label{eq:first_moment}
\end{align}
We simplify each term $A_1,A_2$, and $A_3$:
\begin{align}
    A_1 &= \E\left[\langle \vec{g},\vepsilon^{\,i}\rangle\varepsilon^i_k\right] = \E\left(\sum_{j=1}^d g_j\varepsilon^i_j  \varepsilon^i_k \right) \nonumber \\
        &= g_k\E( \varepsilon^i_k)^2 
        + \sum_{j\neq k}^d g_j\E\left[\varepsilon^i_j\right]\E\left[  \varepsilon^i_k  \right] =g_k\sigma^2
        \label{eq:fm_A1f}\\
    A_2 &= \E\left[\delta^i\varepsilon^i_k\right] =\E\left[\delta^i\right]\E\left[\varepsilon^i_k\right] = 0 \label{eq:fm_A2f} \\
    A_3 &= \E\left[\norm{\vepsilon^{\,i}}^2_2\varepsilon^i_k\right] 
    = \E\left(\sum_{j=1}^{d}(\varepsilon^i_j)^2\varepsilon^i_k\right) \nonumber\\
    &= \E\left[(\varepsilon^i_k)^3\right] 
    + \sum_{j\neq k}^{d}\E\left[(\varepsilon^i_j)^2\right]\E\left[\varepsilon^i_k\right] = 0 \label{eq:fm_A3f}
\end{align}

Substituting~\cref{eq:fm_A1f,eq:fm_A2f,eq:fm_A3f} in \cref{eq:first_moment}, we have the first moment of $v^i_k$:
\begin{equation}\label{eq:final_first_moment}
    \E\left[v^i_k\right] = -g_k\sigma^2 
\end{equation}

\paragraph{Proof of~\cref{eq:second-moment}}
The second moment reads:
\begin{align}
    &\E\left[(v^i_k)^2\right] \nonumber\\
    &=\ \E\left[\left(-\langle \vec{g},\vepsilon^{\,i}\rangle + \delta^{i} + R\norm{\vepsilon^{\,i}}^2_2 + A \right)^2 (\varepsilon^i_k)^2\right] \nonumber \\
    &=\ \underbrace{\E\left[\langle \vec{g},\vepsilon^{\,i}\rangle^2(\varepsilon^i_k)^2\right]}_{B_1}
    + \underbrace{\E\left[(\delta^{i})^2(\varepsilon^i_k)^2\right]}_{B_2} \nonumber \\
    &+ R^2 \underbrace{\E\left[\norm{\vepsilon^{\,i}}^4_2(\varepsilon^i_k)^2\right]}_{B_3} 
    + A^2 \underbrace{\E\left[(\varepsilon^i_k)^2\right]}_{B_4=\sigma^2} \nonumber \\
    &- 2 \underbrace{\E\left[\langle \vec{g},\vepsilon^{\,i}\rangle\delta^{i} (\varepsilon^i_k)^2\right]}_{B_5} 
    - 2R \underbrace{\E\left[\langle \vec{g},\vepsilon^{\,i}\rangle\norm{\vepsilon^{\,i}}^2_2 (\varepsilon^i_k)^2\right]}_{B_6} \nonumber \\
    &+ 2R \underbrace{\E\left[\delta^{i} \norm{\vepsilon^{\,i}}^2_2(\varepsilon^i_k)^2\right]}_{B_7} 
    - 2A \underbrace{\E\left[\langle \vec{g},\vepsilon^{\,i}\rangle(\varepsilon^i_k)^2\right]}_{B_8} \nonumber \\
    &+ 2A \underbrace{\E\left[\delta^i(\varepsilon^i_k)^2\right]}_{B_9} 
    + 2AR \underbrace{\E\left[\norm{\vepsilon^{\,i}}_2^2(\varepsilon^i_k)^2\right]}_{B_{10}} \label{eq:second-moment-expansion}
\end{align}

We simplify each above term:
\begin{align}
    B_1
    &=\E\left(\sum_{i,j=1}^d g_j g_l \varepsilon^i_j \varepsilon^i_{l} (\varepsilon^i_k)^2\right) \nonumber \\
    &= \sum_{j\neq k}^d\sum_{l\neq j,k}^d g_j g_l \E\left[\varepsilon^i_j\right] \E\left[\varepsilon^i_{l}\right] \E\left[(\varepsilon^i_k)^2\right] \nonumber \\
    &\quad + 2\sum_{j\neq k}^d g_j g_k \E\left[\varepsilon^i_j\right] \E\left[(\varepsilon^i_k)^3\right] \nonumber \\
    &\quad + \sum_{j\neq k}^d g_j^2 \E\left[(\varepsilon^i_j)^2\right] \E\left[(\varepsilon^i_k)^2\right] 
    + g_k^2 \E\left[(\varepsilon^i_k)^4\right] \nonumber \\
    &= \left(\norm{\vec{g}}_2^2 + 2 g_k^2\right)\sigma^4 \label{eq:sm_B1f}\\
    B_2 &=  \E\left[(\delta^{i})^2\right] \E\left[(\varepsilon^i_k)^2\right] = \frac{\tau^2\sigma^2 }{M}\label{eq:sm_B2f} \\
    B_3 &= 
    \E\left(\sum_{j=1}^d\sum_{l=1}^d (\varepsilon^i_j)^2 (\varepsilon^i_l)^2 (\varepsilon^i_k)^2\right)\nonumber \\
    &=\sum_{j\neq k}^d\sum_{l\neq j,k}^d \E\left[(\varepsilon^i_j)^2\right]\E\left[(\varepsilon^i_l)^2\right]\E\left[(\varepsilon^i_k)^2\right] \nonumber \\
    &\quad + 2\sum_{j\neq k}^d  \E\left[(\varepsilon^i_j)^2\right]\E\left[(\varepsilon^i_k)^4\right] \nonumber \\
    &\quad + \sum_{j\neq k}^d \E\left[(\varepsilon^i_j)^4\right]\E\left[(\varepsilon^i_k)^2\right] 
    + \E\left[(\varepsilon^i_k)^6\right] \nonumber \\
    &= (d^2+6d+8) \sigma^6 \label{eq:sm_B3f} \\
    B_5 &= \E\left(\sum_{j=1}^d g_j\varepsilon^i_j \; \delta^{i}(\varepsilon^i_k)^2\right) \nonumber \\
    &=\sum_{j=1}^d g_j \E\left[\varepsilon^i_j\right] \E\left[\delta^i\right]\E\left[(\varepsilon^i_k)^2\right] = 0\label{eq:sm_B5f} \\
    B_6 &= \E\left(\sum_{j=1}^d\sum_{l=1}^d g_j \varepsilon^i_j (\varepsilon^i_{l})^2 (\varepsilon^i_k)^2\right) \nonumber \\
    &=\sum_{j\neq k}^d \sum_{l\neq j,k}^d g_j \E\varepsilon^i_j\, \E(\varepsilon^i_{l})^2 \,\E(\varepsilon^i_k)^2 
    \nonumber \\
    &+ \sum_{j\neq k}^d g_j \E\varepsilon^i_j \,\E(\varepsilon^i_k)^4 + \sum_{j\neq k}^d g_j \E(\varepsilon^i_j)^2\, \E(\varepsilon^i_k)^3
    \nonumber \\
    &+ \sum_{j\neq k}^d g_j \E(\varepsilon^i_j)^3\,\E(\varepsilon^i_k)^2 
    + g_k \E(\varepsilon^i_k)^5 = 0 \label{eq:sm_B6f} \\
    B_7 &= \E\left[\delta^{i}\right] \E\left[\norm{\vepsilon^{\,i}}^2_2(\varepsilon^i_k)^2\right] = 0 \label{eq:sm_B7f} \\
    B_8 &= \E\left(\sum_{j=1}^d g_j \varepsilon^i_j (\varepsilon^i_k)^2\right) \nonumber \\
    &= g_k \E(\varepsilon^i_k)^3 + \sum_{j\neq k}^d g_j \E\left[\varepsilon^i_j\right] \E\left[(\varepsilon^i_k)^2\right] =0 \label{eq:sm_B8f} \\
    B_9 &= \E\left[\delta^i\right] \E\left[(\varepsilon^i_k)^2\right]=0 \label{eq:sm_B9f} \\
    \!\!\!\!B_{10} & =\E(\varepsilon^i_k)^4 + \sum_{j\neq k}^d \E(\varepsilon^i_j)^2 (\varepsilon^i_k)^2=(d+2)\sigma^4 \label{eq:sm_B10f}
\end{align}

Substituting~\cref{eq:sm_B1f,eq:sm_B2f,eq:sm_B3f,eq:sm_B5f,eq:sm_B6f,eq:sm_B7f,eq:sm_B8f,eq:sm_B9f,eq:sm_B10f} into~\cref{eq:second-moment-expansion}, we have the the second non-central moment  of $v^i_k$:
\begin{align*}
   \E\left[(v^i_k)^2\right] &=\frac{\tau^2\sigma^2}{M} + \left(\norm{\vec{g}}^2 + 2g_k^2\right)\sigma^4 + A^2\sigma^2  \\
    &+ R^2(d^2+6d+8)\sigma^6 + 2AR(d+2)\sigma^4 
\end{align*}
Ignoring the $\mathcal{O}(\sigma^6)$ term (as commonly $\sigma < 1$) and the remainder $R$ from Taylor expansion, we have:
\begin{equation}\label{eq:final_second_moment}
    \E\left[(v^i_k)^2\right] \approx \frac{\tau^2\sigma^2}{M} + (\norm{\vec{g}}^2 + 2g_k^2)\sigma^4 + A^2\sigma^2.
\end{equation}
\begin{table*}[t]
  \caption{Smooth benchmark functions used in the experiments with their search space, respectively.}
  \label{tab:functions}
  \centering
  \begin{tabular}{l|l|l}
    \toprule
    Name & $\mathcal{L}(\Vec{x})$ & Search Space\\
    \midrule
    Sphere & $\sum_{i=1}^d x_i^2$ & $[-5,5]^d$\\
    Ellipsoid  & $\sum_{i=1}^d 100^{\frac{i-1}{d-1}}x_i^2$ & $[-5,5]^d$\\
    Rotated Ellipsoid  & $\sum_{i=1}^d100^{\frac{d-i}{d-1}} x_i^2$ & $[-5,5]^d$\\
    Hyper-Ellipsoid  & $\sum_{i=1}^d ix_i^2$ & $[-5,5]^d$\\
    Rotated Hyper-Ellipsoid  & $\sum_{i=1}^d (d-i+1)x_i^2$ & $[-5,5]^d$\\
    Rastrigin  & $10d + \sum_{i=1}^d\left[x_i^2 - 10\cos(2\pi x_i) \right]$ & $[-5,5]^d$\\
    Trid  & $\sum_{i=1}^d (x_i-1)^2 - \sum_{i=2}^d x_ix_{i-1}$ & $[-d^2,d^2]^d$\\
    Cosine Mixture  & $-0.1\sum_{i=1}^d \cos(5\pi x_i) +\sum_{i=1}^d x_i^2$ & $[-1,1]^d$\\
    Bohachevsky  & $\sum_{i=1}^{d-1} \left[x_i^2 + 2 x_{i+1}^2 -0.3\cos(3\pi x_i) - 0.4\cos(4\pi x_{i+1}) + 0.7\right]$ & $[-15,15]^d$\\
    Schwefel02 & $\sum_{i=1}^d\left(\sum_{j=1}^i x_i\right)^2$ & $[-10,10]^d$\\
    \bottomrule
  \end{tabular}
\end{table*}
\begin{figure*}[h]
    \centering
    \includegraphics[width=.49\linewidth]{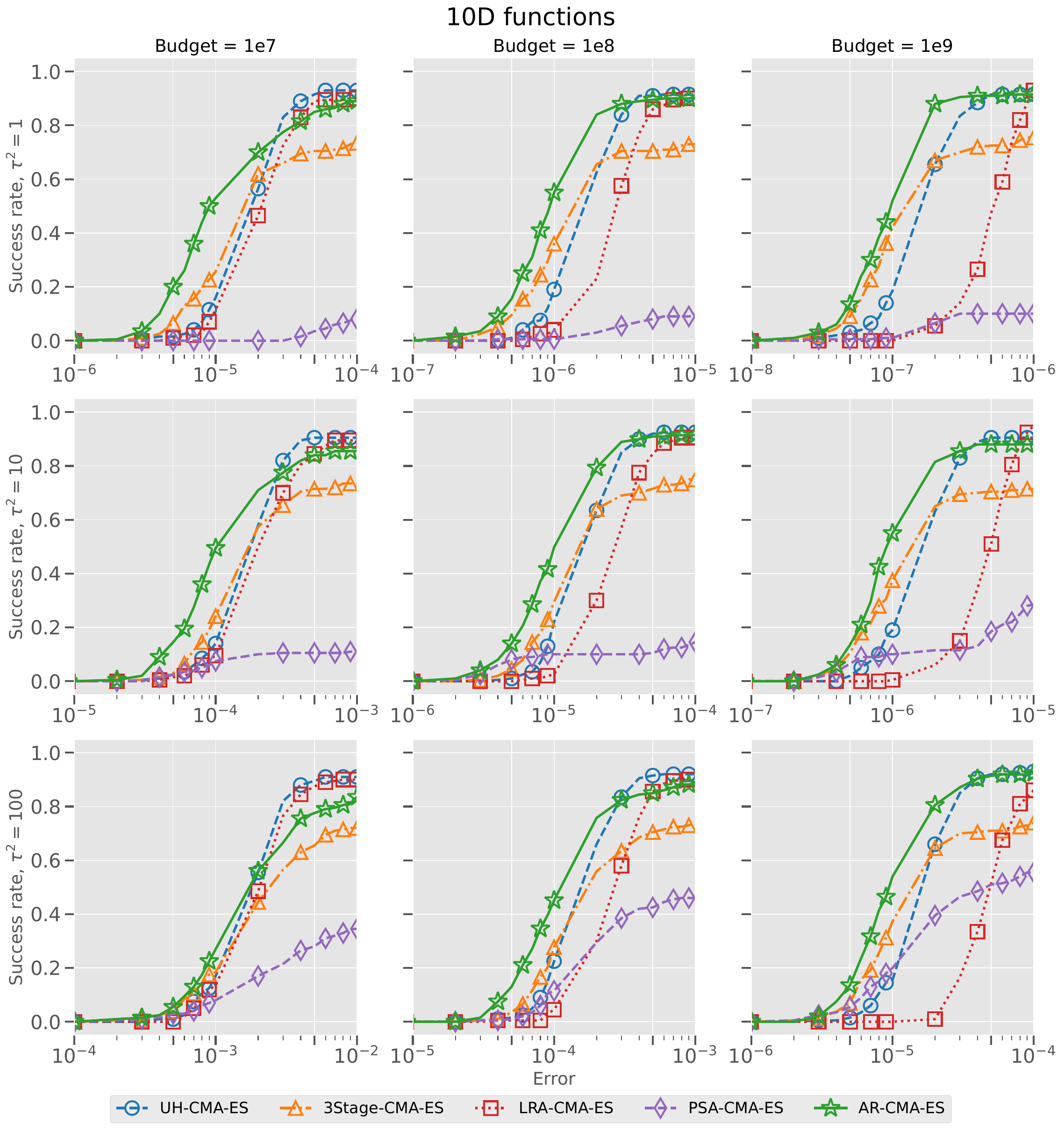}
    \includegraphics[width=.49\linewidth]{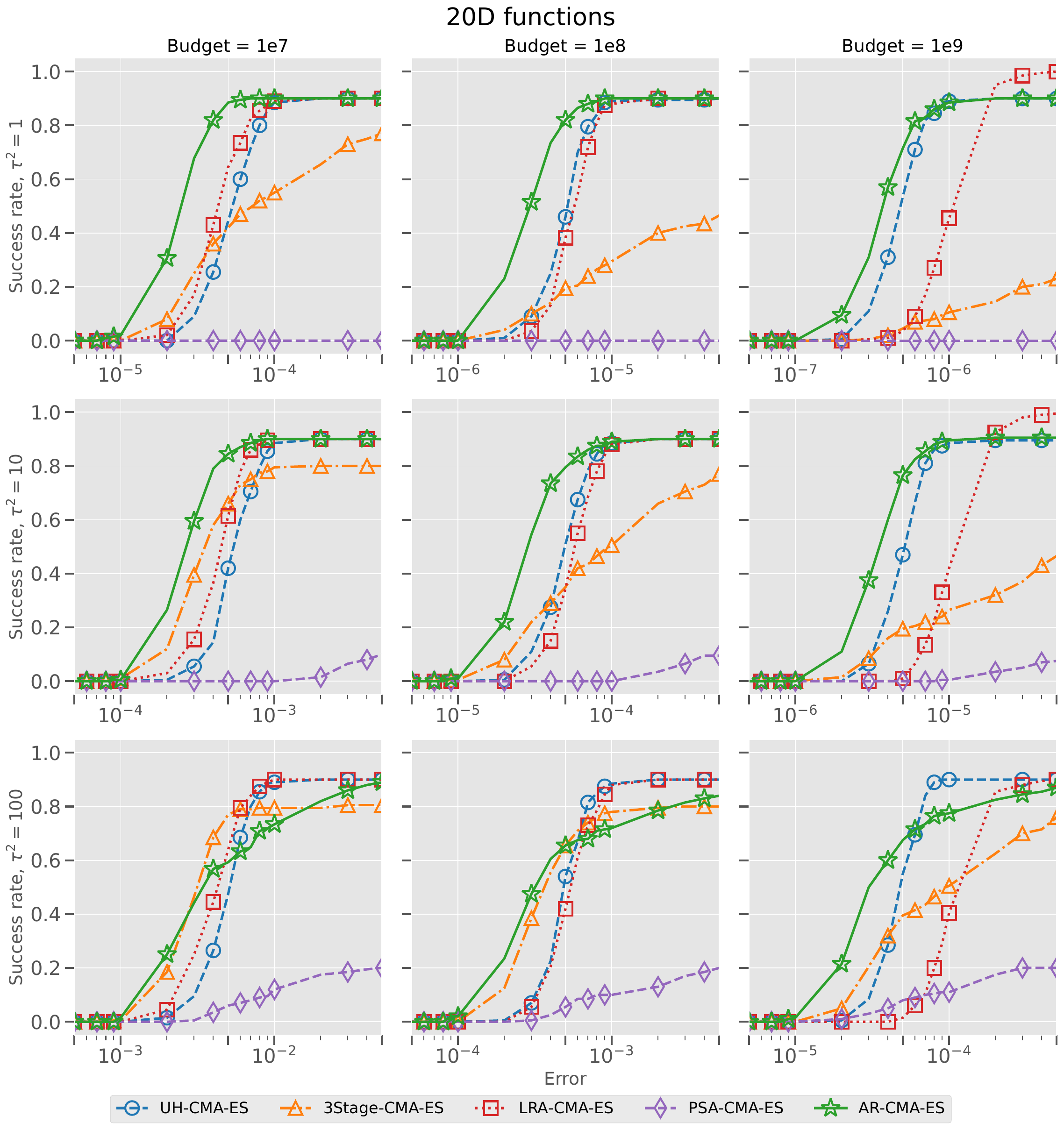}
    \caption{Empirical cumulative distribution functions (ECDFs) of the error ($\loss(\vec{m}) - \loss^*$) aggregated over all test functions are shown for each combination of the noise level ($\tau^2 \in \{1,10,100\}$)  and evaluation budget ($10^7, 10^8, 10^9$). Left: 10-dimensional results; Right: 20-dimensional.}
    \label{fig:all_10D_noise}
\end{figure*}

\section{Lower Bound of the efficiency $\gamma$}\label{sec:lower-bound-improvement}

\paragraph{Proof of~\cref{eq:lower-bound-improvement}}
Taking expectations on both sides of \cref{eq:lower_bound}, we have:
\begin{align}
    &\E\left(\loss(\vtheta) - \loss(\vtheta + \vec{z})\right)  \nonumber\\
    &\geq -\underbrace{\E\left[\langle\vec{g},\vec{v}\,'\rangle\right]}_{C_1}- \frac{Ks_{\text{max}}}{2} \underbrace{\E\left[\norm{\vec{v}\,'}^2_2 \right]}_{C_2}. \label{eq:lb-original}
\end{align}
We simplify terms $C_1$ and $C_2$:
\begin{align}
    C_1 &= \E\left\langle\vec{g},\frac{1}{2\lambda A}\sum_{i=1}^{\lambda}\vec{v}^{\,i}\right\rangle \nonumber \\
    &=\frac{1}{2\lambda A}\sum_{i=1}^{\lambda}\sum_{k=1}^d g_k\E\left[ v^i_k\right] \nonumber \\
    &\labelrel={eq:C_1_e} \frac{1}{2\lambda A}\sum_{i=1}^{\lambda}\sum_{k=1}^d -g_k^2\sigma^2 = -\frac{\sigma^2}{2A}\|\vec{g}\|_2^2 &\label{eq:lb_C1f}
\end{align}
Note that in step~\eqref{eq:C_1_e}, we use the first moment result in ~\cref{eq:final_first_moment}.
\begin{align}
    \!\!\!\!C_2 &=\E\left\langle \frac{1}{2\lambda A}\sum_{i=1}^{\lambda}\vec{v}^{\,i},\frac{1}{2\lambda A}\sum_{j=1}^{\lambda}\vec{v}^{\,j}\right\rangle \nonumber \\
        &=\frac{1}{4\lambda^2 A^2}\sum_{i, j=1}^{\lambda}\sum_{k=1}^d\E\left[ v^i_k v^j_k \right] \nonumber \\
        &=\frac{1}{4\lambda^2 A^2}\left(\sum_{i\neq j}\sum_{k=1}^d\E\left[ v^i_k\right]\E\left[v^j_k \right] + \sum_{i=1}^{\lambda}\sum_{k=1}^d\E\left[ (v^i_k)^2\right] \right)\nonumber \\
        &\labelrel={eq:C_2_f} \frac{1}{4\lambda^2 A^2}\Bigg[\sum_{i\neq j}\sum_{k=1}^d g_k^2\sigma^4 \nonumber\\
        &+\lambda\sum_{k=1}^d \left(\frac{\tau^2\sigma^2}{M} + (\|\vec{g}\|^2 + 2g_k^2)\sigma^4 + A^2\sigma^2\right)\Bigg] \nonumber \\
        &= \frac{\sigma^2 d\tau^2}{4M\lambda A^2}
        + \frac{(\lambda + d + 1)\sigma^4\|\vec{g}\|_2^2 + A^2d\sigma^2}{4\lambda A^2}  \label{eq:lb_C2f}
\end{align}
Note that in step~\eqref{eq:C_2_f}, we use the results from~\cref{eq:final_first_moment,eq:final_second_moment}.

Combining~\cref{eq:lb_C1f,eq:lb_C2f} with~\cref{eq:lb-original}, we have:
\begin{align}
    &\E\left(\loss(\vec{m}) - \loss(\vec{m} + \vec{z})\right) \nonumber \\
    &\geq \frac{\sigma^2}{2A}\|\vec{g}\|^2 
    -\frac{1}{M}\frac{Ks_{\text{max}} \sigma^2d \tau^2}{8\lambda A^2}\nonumber \\
    &\; -\frac{Ks_{\text{max}}(\lambda + d + 1)\sigma^4}{8\lambda A^2} \|\vec{g}\|_2^2 - \frac{dKs_{\text{max}}\sigma^2}{8\lambda } \nonumber \\
    &= \frac{\sigma^2}{2A}\|\vec{g}\|_2^2 
    - \frac{\sigma^4(\lambda+d+1)Ks_{\text{max}}}{8\lambda A^2} \|\vec{g}\|_2^2 \nonumber \\
    &\;-\frac{dKs_{\text{max}}\sigma^2}{8\lambda} -\frac{1}{M}\frac{\sigma^2dKs_{\text{max}}\tau^2}{8\lambda A^2}
\end{align}

\section{Experimental Results on Smooth functions} \label{sec:all_experiments}
We include detailed experimental results here. In Table~\ref{tab:functions}, we list the definitions of the smooth test functions considered in this study. In Fig.~\ref{fig:all_10D_noise}, we show the ECDF curves for each combination of the noise level and evaluation budget. Also, in Fig.~\ref{fig:all_functions_general_plots_10D} and~\ref{fig:all_functions_general_plots_20D}, we include the ECDF on each function for 10-, and 20-dimensional experiments, respectively.

\begin{figure*}
    \centering
    \includegraphics[width=0.8\linewidth]{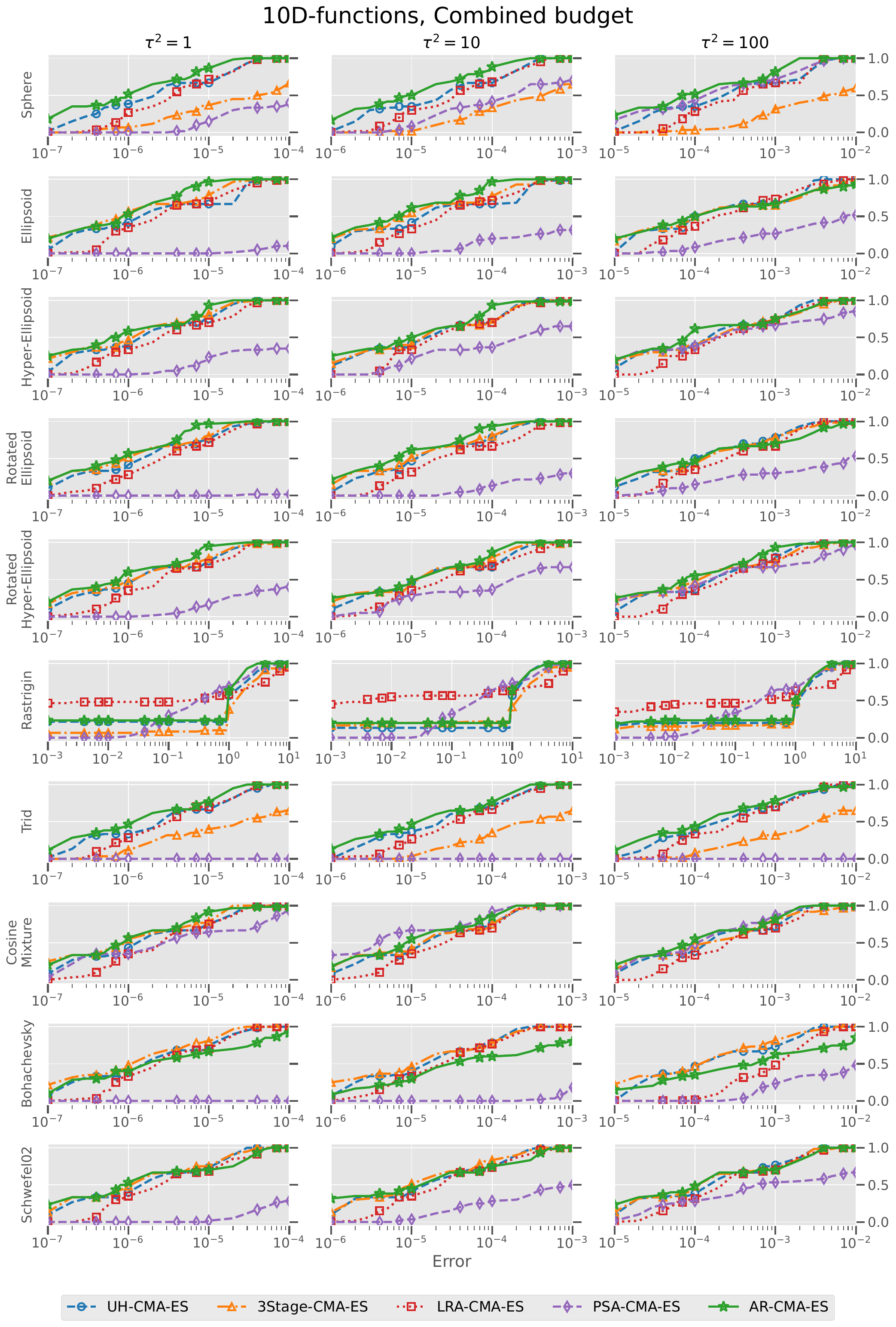}
    \caption{Empirical Cumulative Distribution Function (ECDF) of the optimization error for each 10D function and noise level ($\tau^2 \in \{1, 10, 100\}$). }
    \label{fig:all_functions_general_plots_10D}
\end{figure*}

\begin{figure*}
    \centering
    \includegraphics[width=0.8\linewidth]{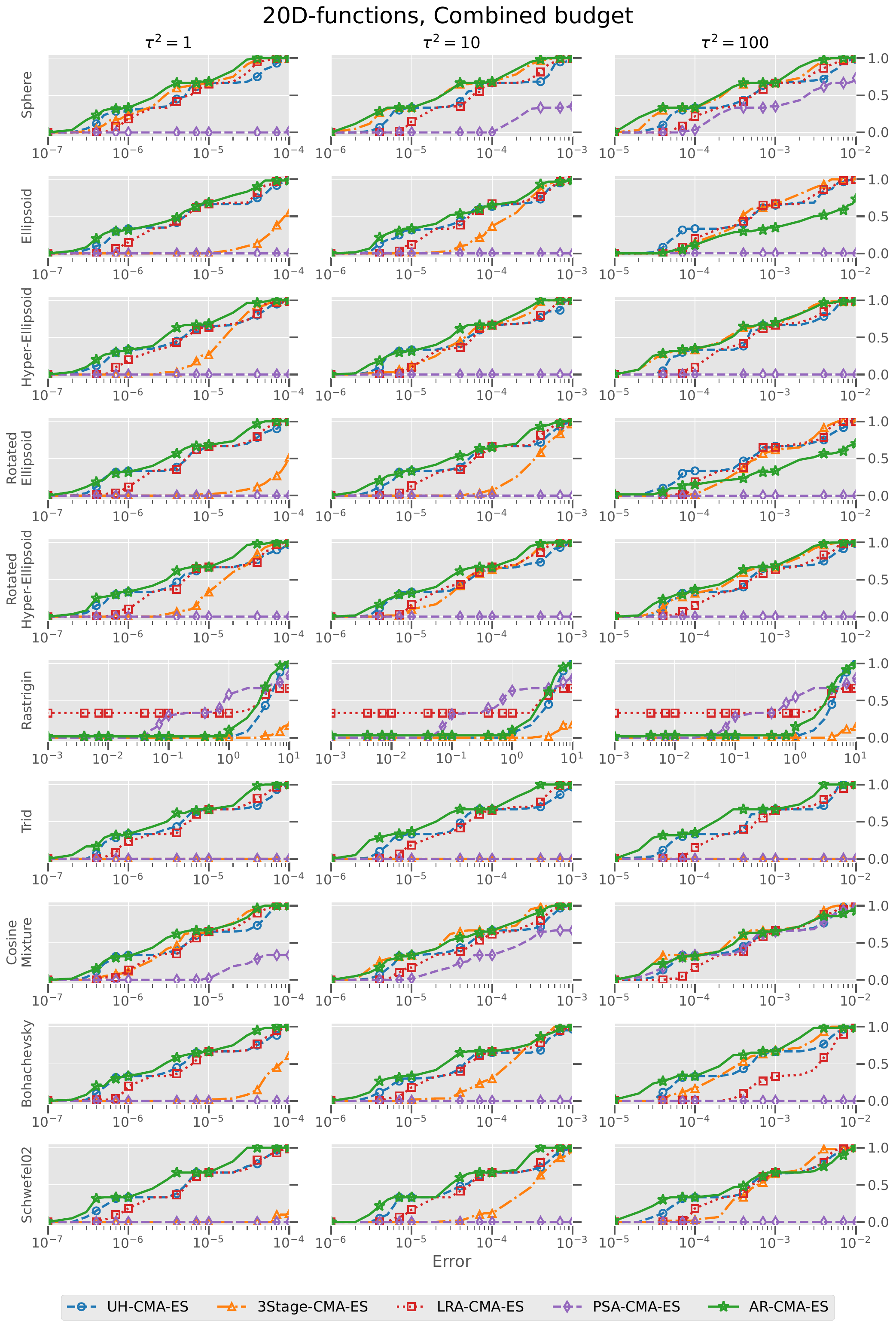}
    \caption{Empirical Cumulative Distribution Function (ECDF) of the optimization error for each 20D function and noise level ($\tau^2 \in \{1, 10, 100\}$).}
    \label{fig:all_functions_general_plots_20D}
\end{figure*}

\section{Experimental Results on Non-smooth functions}\label{sec:nonsmoothresults}
In Table~\ref{tab:non-smooth-functions}, we list the definitions of the non-smooth test functions considered in this study. 
Table~\ref{tab:non-smooth-noiseless} and Table~\ref{tab:non-smooth-noisy} show the performance comparison of AR-CMA-ES with other variants of CMA-ES for non-smooth test functions in noiseless and noisy settings, respectively.
In Fig.~\ref{fig:all_functions_general_plots_10D_non_smooth}, we show the ECDF curves on each function for 10-dimensional experiments for noise level of $\tau^2=1$ and evaluation budget of $10^9$.

\begin{table*}[!htbp]
  \caption{Non-smooth benchmark functions used in the experiments with their search space, respectively.}
  \label{tab:non-smooth-functions}
  \centering
  \begin{tabular}{l|l|l}
    \toprule
    Name & $\mathcal{L}(\Vec{x})$ & Search Space\\
    \midrule
    Sum Absolute & $\sum_{i=1}^{d} \left|x_i\right| $ &  $[-1,1]^d$ \\
    Nesterov F1~\cite{gurbuzbalaban2012nesterov} & $ 0.25 (x_1 - 1)^2 + \sum_{i=2}^{d} \left| x_i - 2x_{i-1}^2 + 1 \right| $ &  $[-1,1]^d$ \\
    Nesterov F2~\cite{gurbuzbalaban2012nesterov} & $ 0.25 \left| x_1 - 1 \right| + \sum_{i=2}^{d} \left| x_i - 2\left| x_{i-1} \right| + 1 \right| $ &  $[-1,1]^d$ \\
    Griewank~\cite{bosse2024piecewise} & $1 + \frac{1}{4000} \sum_{i=1}^{d} x_i^2 - \prod_{i=1}^{d} \left( \left| \cos\left( \frac{x_i}{2\sqrt{i}} \right) \right| - \left| \sin\left( \frac{x_i}{2\sqrt{i}} \right) \right| \right)$ &  $[-1,1]^d$ \\
    
    \bottomrule
  \end{tabular}
\end{table*}

\begin{table*}[!htbp]
\centering
\caption{Numerical verification of AR-CMA-ES against CMA-ES for $d=10$ on noiseless test non-smooth functions ($\tau^2 = 0$) with evaluation budget of $10^9$. For each algorithm, we report the mean and standard error of the final noiseless precision achieved over $20$ runs.\label{tab:non-smooth-noiseless}}
\begin{tabular}{|l|c|c|}
\hline
\textbf{Problem} & \textbf{AR-CMA-ES} & \textbf{CMA-ES} \\
\hline
Sum Absolute & $1.58 \times 10^{-10} \pm 1.55 \times 10^{-10}$ & $3.24 \times 10^{-11} \pm 2.30 \times 10^{-11}$ \\
\hline
Nesterov F1  & $1.85 \times 10^{-1} \pm 2.88 \times 10^{-1}$    & $2.14 \times 10^{-1} \pm 3.15 \times 10^{-1}$    \\
\hline
Nesterov F2  & $1.66 \times 10^{-1} \pm 1.11 \times 10^{-1}$    & $2.03 \times 10^{-1} \pm 1.18 \times 10^{-1}$    \\
\hline
Griewank    & $5.64 \times 10^{-7} \pm 2.88 \times 10^{-7}$    & $3.62 \times 10^{-7} \pm 1.33 \times 10^{-7}$    \\
\hline
\end{tabular}
\end{table*}

\begin{figure*}[!htbp]
    \centering
    \includegraphics[width=0.85\linewidth]{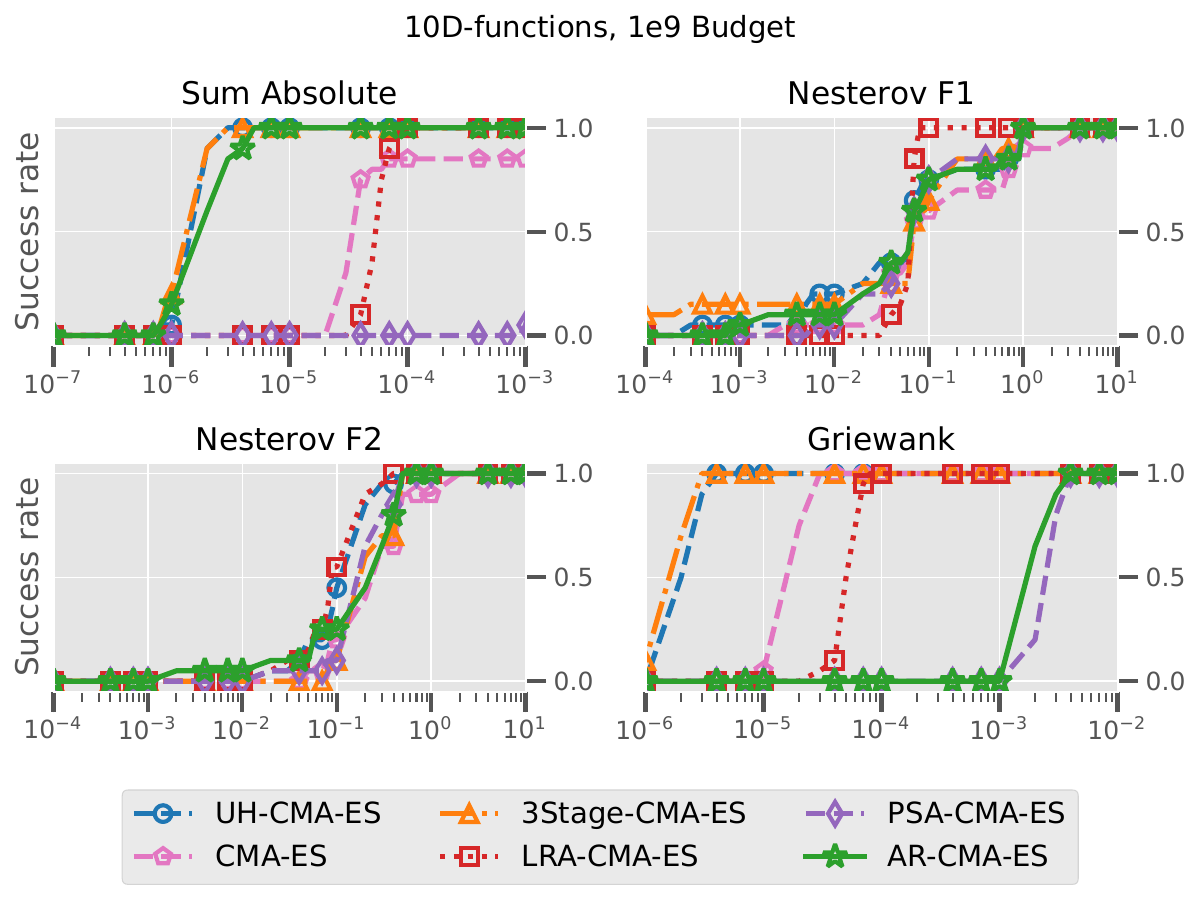}
    \caption{Empirical Cumulative Distribution Function (ECDF) of the optimization error for each 10D non-smooth function and noise level $\tau^2 = 1$.}
    \label{fig:all_functions_general_plots_10D_non_smooth}
\end{figure*}
\begin{table*}[!htbp]
\centering
\caption{Performance comparison across non-smooth benchmark functions (columns) for different CMA variants (rows) under noise ($\tau^2 = 1$) for $d=10$ and evaluation budget $10^9$. For each algorithm, we report the mean and standard error of the final noiseless precision achieved over $20$ runs.\label{tab:non-smooth-noisy}}
\begin{tabular}{|l|c|c|c|c|}
\hline
\textbf{Algorithm} & \textbf{Sum Absolute} & \textbf{Nesterov F1} & \textbf{Nesterov F2} & \textbf{Griewank} \\
\hline
AR-CMA-ES   & $1.97 \times 10^{-6} \pm 9.81 \times 10^{-7}$ & $2.16 \times 10^{-1} \pm 3.47 \times 10^{-1}$ & $2.25 \times 10^{-1} \pm 1.59 \times 10^{-1}$ & $1.97 \times 10^{-3} \pm 7.03 \times 10^{-4}$ \\
\hline
CMA-ES      & $4.22 \times 10^{-3} \pm 1.23 \times 10^{-2}$ & $5.10 \times 10^{-1} \pm 9.82 \times 10^{-1}$ & $3.65 \times 10^{-1} \pm 4.24 \times 10^{-1}$ & $1.64 \times 10^{-5} \pm 4.55 \times 10^{-6}$ \\
\hline
3-Stage-CMA-ES  & $1.37 \times 10^{-6} \pm 4.48 \times 10^{-7}$ & $1.79 \times 10^{-1} \pm 2.86 \times 10^{-1}$ & $2.44 \times 10^{-1} \pm 1.47 \times 10^{-1}$ & $1.65 \times 10^{-6} \pm 4.78 \times 10^{-7}$ \\
\hline
UH-CMA-ES   & $1.46 \times 10^{-6} \pm 4.46 \times 10^{-7}$ & $2.15 \times 10^{-1} \pm 3.47 \times 10^{-1}$ & $1.40 \times 10^{-1} \pm 1.05 \times 10^{-1}$ & $2.08 \times 10^{-6} \pm 5.85 \times 10^{-7}$ \\
\hline
LRA-CMA-ES  & $5.50 \times 10^{-5} \pm 1.45 \times 10^{-5}$ & $6.12 \times 10^{-2} \pm 1.12 \times 10^{-2}$ & $1.18 \times 10^{-1} \pm 8.39 \times 10^{-2}$ & $5.15 \times 10^{-5} \pm 1.09 \times 10^{-5}$ \\
\hline
PSA-CMA-ES  & $1.66 \times 10^{-3} \pm 7.11 \times 10^{-4}$ & $1.93 \times 10^{-1} \pm 3.28 \times 10^{-1}$ & $2.11 \times 10^{-1} \pm 1.25 \times 10^{-1}$ & $2.42 \times 10^{-3} \pm 7.18 \times 10^{-4}$ \\
\hline
\end{tabular}
\end{table*}

\end{document}